\documentclass{article}

\usepackage[nonatbib,final]{neurips_2021}

\usepackage[utf8]{inputenc} 
\usepackage[T1]{fontenc}    
\usepackage{hyperref}       
\usepackage{url}            
\usepackage{booktabs}       
\usepackage{amsfonts}       
\usepackage{nicefrac}       
\usepackage{tcolorbox}
\usepackage{microtype}      
\usepackage{xcolor}         
\usepackage{hyperref}
\usepackage{url}
\usepackage{caption}
\usepackage{array}
\usepackage{makecell}
\usepackage{tabularx}
\usepackage{amsmath}
\usepackage{amsthm}
\usepackage{amssymb}
\usepackage{mathtools}      
\usepackage{microtype}      
\usepackage{nicefrac}       
\usepackage{subcaption} 
\usepackage{url}  

\usepackage{xspace}
\usepackage{amssymb}
\usepackage{bm}

\usepackage{multirow}
\usepackage{enumitem}

\usepackage{amsthm}
\usepackage{bbm}
\usepackage[numbers]{natbib} 
\usepackage{xcolor}
 \usepackage{appendix}

\usepackage{wrapfig}

\usepackage{thmtools,thm-restate}

\newtheorem{Thm}{Theorem}

\newtheorem{Df}{Definition}

\newtheorem{Coro}[Thm]{Corollary}

\DeclareMathOperator*{\E}{\mathbb{E}}

\newcommand{\BK}[1]{\textcolor{red}{BK: #1}}
\newcommand{\rwd}{\text{rewinding}\xspace}
\newcommand{\RWD}{\text{Rewinding}\xspace}
\newcommand{\init}{\text{initialization}\xspace}
\newcommand{\INIT}{\text{Initialization}\xspace}
\newcommand{\posthoc}{\text{traditional}\xspace}

\newcommand{\cnet}{\texttt{CARD}\xspace}
\newcommand{\cnets}{\texttt{CARDs}\xspace}

\newcommand{\cdeck}{\texttt{CARD-Deck}\xspace}
\newcommand{\cdecks}{\texttt{CARD-Decks}\xspace}

\title{{\em A Winning Hand}: Compressing Deep Networks Can Improve Out-Of-Distribution Robustness}

\author{James Diffenderfer, Brian R. Bartoldson\thanks{equal contribution}, Shreya Chaganti\setcounter{footnote}{0}\footnotemark, Jize Zhang, Bhavya Kailkhura \\
  Lawrence Livermore National Laboratory\\
  \texttt{\{diffenderfer2, bartoldson, chaganti1, zhang64, kailkhura1\}@llnl.gov} \\
}

\begin{document}

\maketitle

\begin{abstract}
Successful adoption of deep learning (DL) in the wild requires models to be: (1) compact, (2) accurate, and (3) robust to distributional shifts. Unfortunately, efforts towards simultaneously meeting these requirements have mostly been unsuccessful.
This raises an important question: {\em ``Is the inability to create \texttt{C}ompact, \texttt{A}ccurate, and \texttt{R}obust \texttt{D}eep neural networks (\cnets) fundamental?''} 
To answer this question, we perform a large-scale analysis of popular model compression techniques which uncovers several intriguing patterns. 
Notably, in contrast to traditional pruning approaches (e.g., fine tuning and gradual magnitude pruning), we find that ``lottery ticket-style'' approaches can surprisingly be used to produce \cnets, including binary-weight \cnets. 
Specifically, we are able to create extremely compact \cnets that, compared to their larger counterparts, have similar test accuracy and matching (or better) robustness---simply by pruning and (optionally) quantizing.
Leveraging the compactness of \cnets, we develop a simple domain-adaptive test-time ensembling approach (\cdeck) that uses a gating module to dynamically select appropriate \cnets from the \cdeck based on their spectral-similarity with test samples. 
The proposed approach builds a ``winning hand'' of \cnets that establishes a new state-of-the-art \citep{croce2020robustbench} on CIFAR-10-C accuracies (i.e., \textit{96.8\% standard and 92.75\% robust}) and CIFAR-100-C accuracies (i.e., \textit{80.6\% standard and 71.3\% robust}) with better memory usage than non-compressed baselines 
(pretrained \cnets available at \cite{croce2020robustbench}). 
Finally, we provide theoretical support for our empirical findings. 
\end{abstract}

\section{Introduction}

Deep Neural Networks (DNNs) have achieved unprecedented success in a wide range of applications due to their remarkably high accuracy~\citep{goodfellow2016deep}. 
However, this high performance stems from significant growth in DNN model size; i.e., massive overparameterization. 
Furthermore, these highly overparameterized models are known to be susceptible to the out-of-distribution (OOD) shifts encountered during their deployment in the wild~\cite{bulusu2020anomalous}. 
This resource-inefficiency and OOD brittleness of state-of-the-art (SOTA) DNNs severely limits the potential applications DL can make an impact on. 

For example, consider the ``Mars rover mission'' that uses laser-induced breakdown spectroscopy (LIBS) to search for microbial life. It is well accepted that endowing the rover with DNNs to analyze complex LIBS spectra could produce scientific breakthroughs~\citep{bhardwaj2021semi}. However, employing DNNs in such circumstances is challenging:
1) as these devices are battery operated, the model has to be lightweight so it consumes less memory with reduced power consumption, and 2) the model must be able to efficiently handle domain shifts in spectra caused by environmental noise.
These requirements are not specific only to the aforementioned use case but arise in any resource-limited application using DL in the wild. 
The fact that SOTA DNNs do not satisfy {\em compactness} and {\em OOD robustness} requirements is holding us back from leveraging advances in DL to make scientific discoveries.

This work is driven by two questions around this crucial problem.
{\textbf{Q1.}} {\em Can we show the existence of compact, accurate, and robust DNNs (\cnets)?} 
{\textbf{Q2.}} If yes, {\em can we apply existing robustness improvement techniques to \cnets to further amplify performance while maintaining compactness?}

Notably, there have been some recent successes in addressing each of the challenges \cnets present in isolation. The authors in \citep{hendrycks2019augmix, kireev2021effectiveness} developed data augmentation methods for achieving high OOD robustness without sacrificing accuracy on the clean data. The authors in \citep{frankle2018lottery, diffenderfer2021multiprize} developed pruning and quantization approaches for achieving high accuracy at extreme levels of model compression. 
However, efforts towards achieving model compactness, high accuracy, and OOD robustness simultaneously have mostly been unsuccessful. 
For example, \cite{hooker2019compressed} and \cite{liebenwein2021lost} recently showed that compressed DNNs achieve accuracies similar to the original networks' but are far more brittle when faced with OOD data. Perhaps unsurprisingly, the current solution in the robust DL community to improve the OOD robustness (and accuracy) is to increase the model size (e.g., \citep{hendrycks2020many, gowal2020uncovering,croce2020robustbench}).

In this paper, we demonstrate that these negative results are a byproduct of inapt compression strategies, and the inability to create \cnets is not fundamental (answering \textbf{Q1} in the affirmative). Specifically, we perform a large-scale comparison by varying architectures, training methods, and pruning rates for a range of compression techniques.
We find that in contrast to traditional pruning methods (e.g., fine tuning~\cite{han2015learning} and gradual magnitude pruning~\cite{zhu2017prune}), ``lottery ticket-style'' compression approaches~\cite{frankle2018lottery, renda2020comparing, ramanujan2019whats, diffenderfer2021multiprize} can surprisingly be used to create \cnets. 
In other words, we are able to create extremely compact (i.e., sparse and, optionally, binary) \cnets that are significantly more robust compared to their larger and full-precision counterparts while having comparable test accuracy. 
Our results are in sharp contrast to the existing observation that compression is harmful to OOD robustness. 
In fact, we show that compression is capable of providing improved robustness.

We subsequently explore the possibility of using existing robustness-improvement techniques in conjunction with compression strategies to further improve the performance of \cnets. 
{Empirically, we show the compatibility of \cnets with popular existing strategies, such as data augmentation and model size increase.} 
We also propose a new robustness-improvement strategy that leverages the compactness of \cnets via ensembling---this ensembling approach is referred to as a domain-adaptive \cdeck and uses a gating module to dynamically choose appropriate \cnets for each test sample such that the spectral-similarity of the chosen \cnets and test data is maximized.
This proposed adaptive ensembling approach builds a ``winning hand" of \cnets that establishes a new state-of-the-art robustness (and accuracy) on the popular OOD benchmark datasets CIFAR-10-C and CIFAR-100-C with a compact ensemble~\cite{croce2020robustbench} (answering \textbf{Q2} in the affirmative).

Broad implications of our findings are as follows. 
First, there exist sparse networks early in training (sometimes at initialization) that can be trained to become \cnets (i.e., we extend the lottery ticket hypothesis \cite{frankle2018lottery} to robust neural networks via our \cnet hypothesis).
{Second, within a random-weight neural network, there exist \cnets which (despite having untrained weights) perform comparably to more computationally expensive DNNs  (i.e., we extend the \textit{strong} lottery ticket hypothesis \cite{ramanujan2019whats} to robust neural networks via our \cnet hypothesis).}
{Third, compression can be complementary with existing robustness-improving strategies, which suggests that appropriate compression approaches should be considered whenever training models that may be deployed in the wild.}
To summarize our main contributions:
\begin{itemize}
    \item Contrary to most prior results in the literature, we show that compression can improve robustness, providing evidence via extensive experiments on benchmark datasets, 
    supporting our \cnet hypothesis.  
    Our experiments suggest ``lottery ticket-style'' pruning methods and a sufficiently overparameterized model are key factors for producing \cnets (Section \ref{sec:c10-experiments}).
    \item As corruptions in benchmark datasets could be limited (or biased) to certain frequency ranges, we tested the ability of compression to improve robustness to Fourier basis perturbations. This analysis corroborates findings on CIFAR-10/100-C and further highlights that models compressed via different methods have different robustness levels (Section \ref{sec:fourier}).
    \item Leveraging the compactness of \cnets, we develop a test-time adaptive ensembling method, called a domain-adaptive \cdeck, that utilizes \cnets trained with existing techniques for improving OOD robustness. Resulting models set a new SOTA performance \citep{croce2020robustbench} on CIFAR-10-C and CIFAR-100-C while maintaining compactness (Section \ref{sec:sota}).  
    \item Finally, we provide theoretical support for the \cnet hypothesis and the robustness of the domain-adaptive \cdeck ensembling method (Section \ref{sec:theory}). 
\end{itemize}

\section{Is the inability to create CARDs fundamental?} \label{sec:experiments}
Recent studies of the effects of model compression on OOD robustness have been mostly negative.
For instance, \citet{hooker2019compressed} showed that gradual magnitude pruning \citep{zhu2017prune} of a ResNet-50 \citep{he2015deep} trained on ImageNet \citep{russakovsky2015imagenet} caused accuracy to decrease by as much as 40\% on corrupted images from ImageNet-C \citep{hendrycks2019benchmarking}, while performance on the non-corrupted validation images remained strong. \citet{liebenwein2021lost} reported similar findings for different pruning approaches \citep{renda2020comparing,baykal2019sipping} applied to a ResNet-20 model \citep{he2015deep} tested on an analogously corrupted dataset, CIFAR-10-C.
Consistent with these findings, \citet{hendrycks2020many} found that increasing model size tended to improve robustness.\footnote{For additional discussion of this and more related work, please see Appendix~\ref{sec:background}.}

Critically, these studies suggest that model compression may be at odds with the simultaneous achievement of high accuracy and OOD (natural corruption) robustness. 
However, it's possible that these negative results are a byproduct of inapt compression strategies and/or insufficient overparameterization of the network targeted for compression. 
As such, to motivate the scientific question of interest and our empirical/theoretical analyses, we propose the following alternative hypothesis.



\noindent
\begin{tcolorbox}[colframe=black,colback=lightgray!30,boxrule=0.5pt,boxsep=2pt,left=1pt,right=1pt,top=0pt,bottom=0pt]
\noindent \textbf{{CARD} Hypothesis.} \emph{Given a sufficiently overparameterized neural network, a suitable model compression (i.e., pruning and binarization) scheme can yield a compact network with comparable (or higher) accuracy and robustness than the same network when it is trained without compression.}
\end{tcolorbox}

\subsection{Model compression approaches}
For a comprehensive analysis of existing pruning methods, we introduce a framework inspired by those in \citep{renda2020comparing,wang2021emerging} that covers traditional-through-emerging pruning methodologies. 
Broadly, this framework places a pruning method into one of three categories: (a) traditional, (b) rewinding-based lottery ticket, and (c) initialization-based (strong) lottery ticket. 
Specific pruning methods considered in these respective categories are:
(a) fine-tuning and gradual magnitude pruning, (b) weight rewinding and learning rate rewinding, and (c) edgepopup and biprop. 
Precise definitions of these pruning methods and discussion of differences are available in Appendix~\ref{sec:compression}.

Briefly, fine-tuning (FT) \citep{han2015learning} prunes models once at the end of normal training, then fine-tunes the models for a given number of epochs to recover accuracy lost due to pruning; while gradual magnitude pruning (GMP) \citep{zhu2017prune} prunes models throughout training.
Weight rewinding (LTH) \citep{frankle2018lottery,frankle2020linear} is iterative like GMP but fully trains the network, prunes, rewinds the unpruned weights (and learning rate schedule) to their values early in training, then fully trains the subnetwork before pruning again; learning rate rewinding (LRR) \citep{renda2020comparing} is identical to LTH, except only the learning rate schedule is rewound, not the unpruned weights. Finally, edgepopup (EP) \citep{ramanujan2019whats} does not weight-train the network but instead prunes a randomly initialized network, using training data to find weights whose removal improves accuracy (notably, EP can, and does here, operate on signed initialization); biprop (BP) \citep{diffenderfer2021multiprize} proceeds similarly but incorporates a binarization scheme resulting in a binary-weight network regardless of the initialization used. For all of these methods, we make use of global unstructured pruning, which allows for different pruning percentages at each layer of the network. For BP and EP, we additionally consider layerwise pruning, which prunes the same percentage across all layers. 
We use the hyperparameters specifically tuned for each approach; see Appendix~\ref{appendix:experiment-details} for additional details.

\subsection{Accuracy-robustness comparison of global pruning methods} \label{sec:c10-experiments}
To test the \cnet hypothesis, we use: five models 
(VGG \cite{frankle2018lottery,simonyan2014very} and ResNet \cite{he2015deep} style architectures of varying size), five sparsity levels (50\%, 60\%, 80\%, 90\%, 95\%), and six model compression methods (FT, GMP, LTH, LRR, EP, BP). 
For each model, sparsity level, and compression method, five realizations are trained on the CIFAR-10 training set~\cite{krizhevsky2009learning}. 
Model accuracy and robustness are measured using top-1 accuracy on the CIFAR-10 and CIFAR-10-C test sets, respectively. CIFAR-10-C contains 15 different common corruptions from four categories: noise, blur, weather, and digital corruptions \cite{hendrycks2019benchmarking}. As a baseline, we train 5 realizations of each model without compression. 

In Figure~\ref{fig:card-net-all-plot}, we plot our experimental results. 
Accuracy and robustness values are averaged over the five realizations and plotted relative to the average non-compressed baseline performance. The y-axis measures relative difference in percentage points.
The mean baseline accuracy and robustness for each architecture is listed as the reference accuracy in each plot.
The first row of plots indicate accuracy (top-1 accuracy on CIFAR-10 relative to baseline) while the second row indicate robustness (top-1 accuracy on CIFAR-10-C relative to baseline). 
At each sparsity level, error bars extend to the minimum and maximum relative percentage point difference across all realizations.

\begin{figure}[ht]
    \centering
    \includegraphics[width=\textwidth,clip]{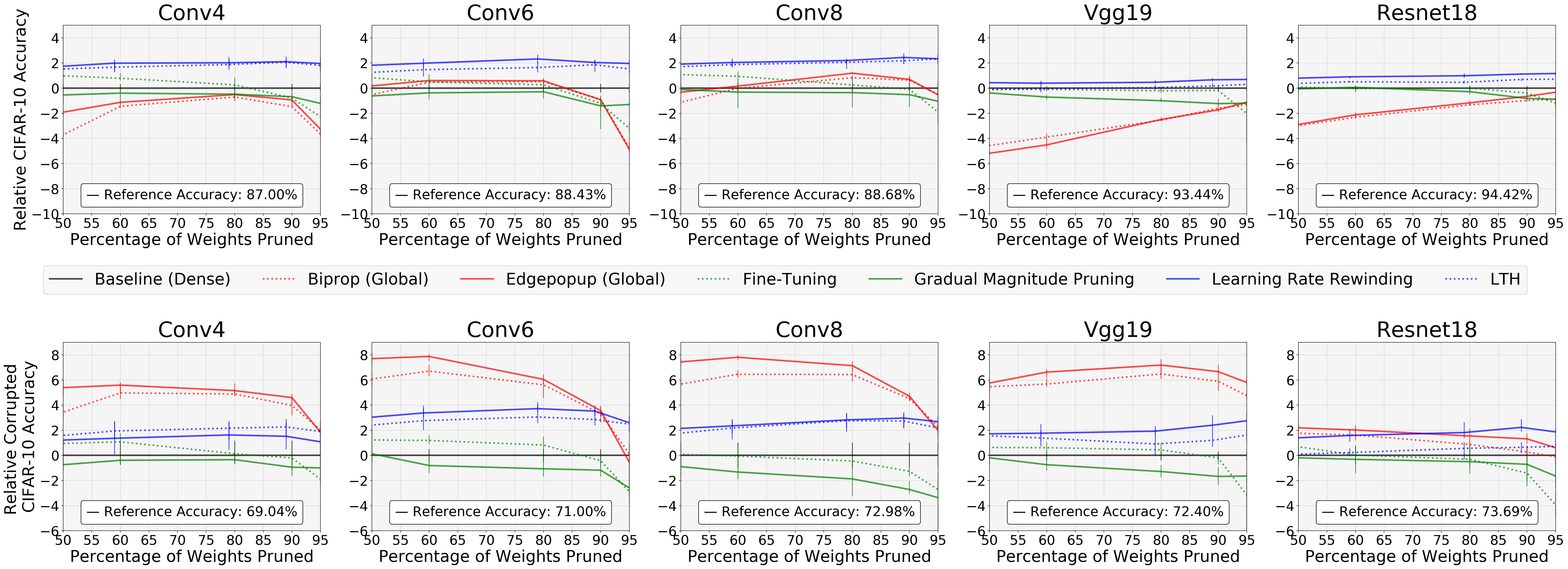}
    \caption{{\bfseries Suitable pruning approaches can improve robustness over dense models}: Comparing the Top-1 accuracy of pruned models relative to the average of dense baseline models on CIFAR-10 and CIFAR-10-C demonstrates that \cnets exist and can be produced using LRR, LTH, BP, or EP.}
    \label{fig:card-net-all-plot}
\end{figure}

Our results for \posthoc methods, i.e., Fine-Tuning and Gradual Magnitude Pruning, are consistent with previous works \citep{hooker2019compressed, liebenwein2021lost} as the robustness of models pruned using these methods degrades relative to the dense models', particularly in higher pruning regimes. 
However, we find that \rwd and \init based pruning approaches consistently produce notable gains in robustness relative to dense baselines while matching (and sometimes surpassing) the accuracy of the dense baseline. 
In particular, the \rwd class of methods provide a consistent, moderate improvement to both accuracy and robustness while the \init class provides more substantial gains in robustness even when the accuracy is slightly below the baseline accuracy. 
The significance of overparameterization to finding highly compact \cnets using \init methods is evident for all architecture types, as the robustness of these models in higher pruning regimes improves at increasing levels of parameterization {for a given architecture class}. 
However, even in models with fewer parameters, we find that \init methods are able to provide notable robustness gains.

Additional experiments involving \init methods are provided in Appendix~\ref{appendix:init-experiments}. Specifically, a comparison of the performance EP and BP using layerwise and global pruning is performed in Section~\ref{appendix:layer-vs-global} and
a comparison of full-precision and binary-weight EP models in Section~\ref{appendix:EP_full_precision}. Empirical results in Section~\ref{appendix:EP_full_precision} indicate that robustness gains provided by EP- and BP-pruned models may be a feature of \init pruning methods and not solely due to weight-binarization. 

\subsection{Viewing the effect of compression on OOD robustness through a spectral lens} \label{sec:fourier}
As CIFAR-10-C corruptions are limited to certain frequency ranges and combinations \cite{yin2019fourier}, it is of interest to validate if the robustness effects of different pruning methods observed in Section~\ref{sec:c10-experiments} hold on a broader ranges of  frequencies. 
To this end, we perform a frequency-domain analysis by utilizing the Fourier sensitivity method \cite{yin2019fourier}, which we briefly summarize below.

Given a model and a test dataset, each image in the test dataset is perturbed using additive noise in the form of 2D Fourier basis matrices, denoted by $U_{i,j} \in \mathbb{R}^{d1\times d2}$. 
Specifically, for an image $X$  and a 2D Fourier basis matrix $U_{i,j}$ a perturbed image is computed by $X_{i,j} = X + r \varepsilon U_{i,j}$, where $r$ is chosen uniformly at random from $\{-1, 1\}$ and $\varepsilon > 0$ is used to scale the norm of the perturbation. 
Note that each channel of the image is perturbed independently. 
Given a set of test images, each Fourier basis matrix can be used to generate a perturbed test set of images on which the test error for the model is measured.
Plotting the error rates as a function of frequencies $(i, j)$ yields the Fourier error heatmap of a model -- a visualization of the sensitivity of a model to different frequency perturbations in the Fourier domain.
Informally, the center of the heat map contains perturbations corresponding to the lowest frequencies 
and the edges correspond to the highest frequencies.

We generate heatmaps for models corresponding to each pruning method as well as layerwise pruned models using EP and BP. 
The norm of the perturbation, $\varepsilon$, is varied over the set $\{3, 4, 6\}$ to represent low, medium, and high levels of perturbation severity. As a reference, we include heatmaps for the dense (non-compressed) baseline model. 
Fourier heatmaps for the Conv8 architecture at 80\% prune percentage are provided in Figure~\ref{fig:fourier-conv8} while additional heatmaps can be found in Section~\ref{appendix:all_heatmaps}. 

\begin{figure}[ht]
\centering
  \begin{tabular}{@{}c@{}}
    \includegraphics[width=\textwidth]{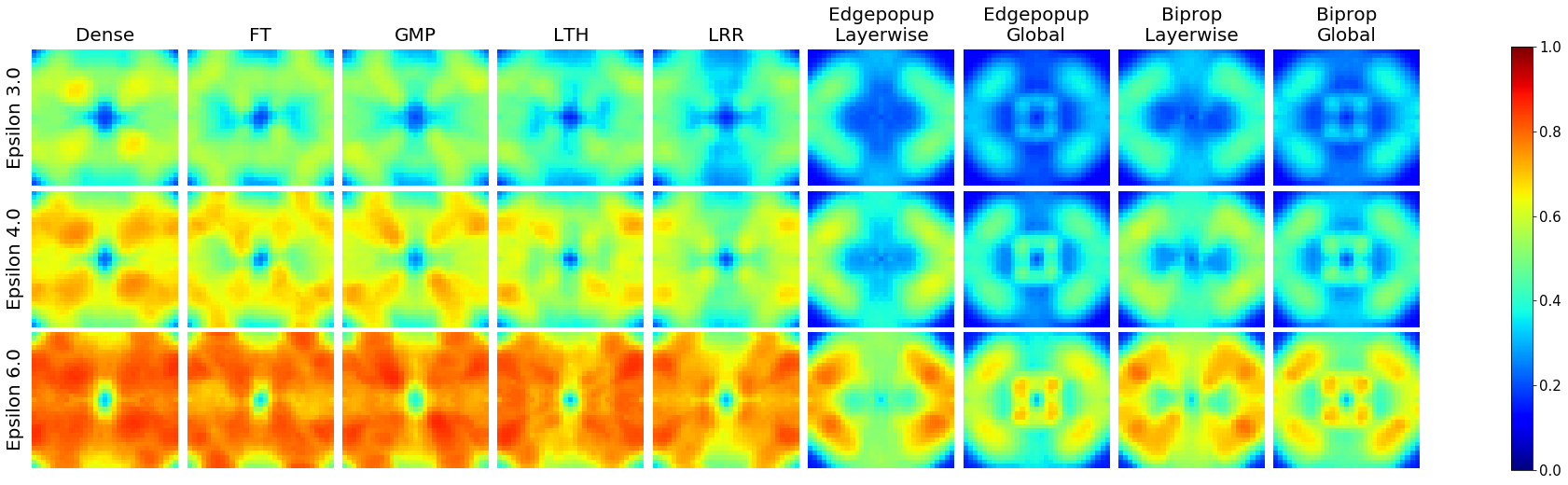} \\
    \includegraphics[width=\textwidth]{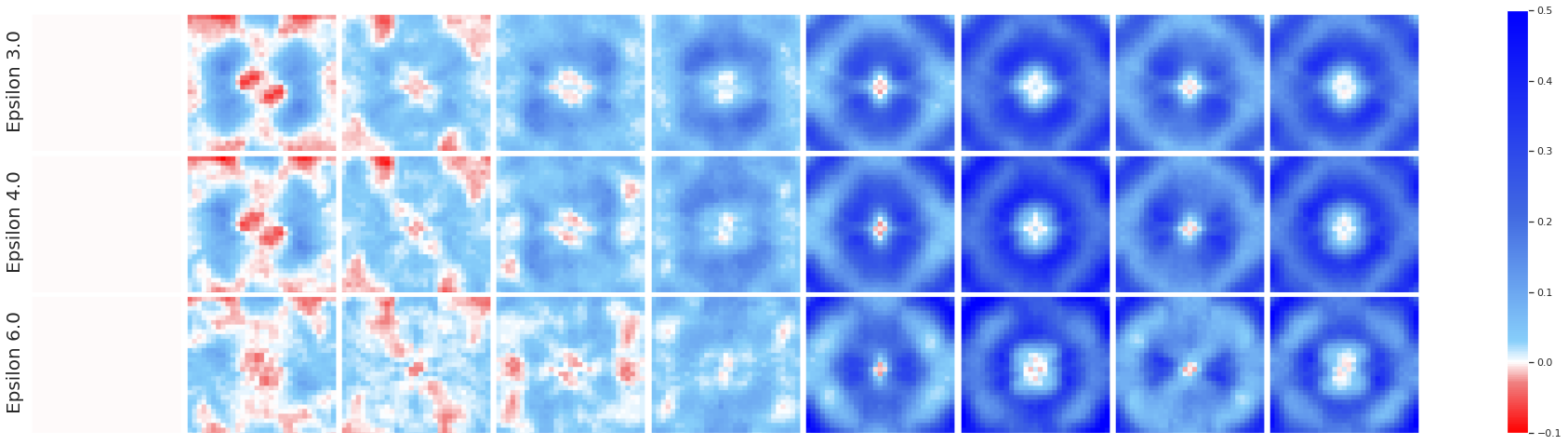}
  \end{tabular}
  \caption{{\bfseries Visualizing the response of compressed models to perturbations at different frequencies}: The top three rows are Fourier heatmaps 
  for Conv8 trained on CIFAR-10 with 80\% of weights pruned. The bottom three rows are difference to the baseline with blue regions indicating lower error rate than baseline. Init. methods provide up to a 50 percentage point improvement in some instances.}
  \label{fig:fourier-conv8}
\end{figure}

Figure~\ref{fig:fourier-conv8} illustrates that \init pruning methods reduce the error rate across nearly the full spectrum of Fourier perturbations relative to the dense model. 
Additionally, \init pruning methods using layerwise pruning present a different response, or error rate, at certain frequency corruptions when compared to heatmaps of global \init pruning methods. 
The difference heatmaps show that \rwd methods offer mild to moderate improvements across much of the frequency spectrum with LRR outperforming LTH in a few regions of the heatmap.
The difference heatmaps also highlight that \posthoc methods degrade the robustness to more Fourier perturbations than other compression methods and result in an increased error rate of 10 percentage points (relative to the dense baseline) in some cases. 
These findings further suggest that the robustness of a compressed model is dependent on the compression method used or the resulting structure of the sparsity.
In Appendix~\ref{appendix:n_shot}, we provide additional heatmaps to examine the impact on robustness when varying the number of rewinding steps 
used by \rwd methods.

To summarize, we have empirically verified our \cnet hypothesis by demonstrating that ``lottery ticket-style'' compression methods can produce compact models with accuracy and robustness comparable to (or higher than) their dense counterparts.

\section{Creating a winning hand of CARDs} \label{sec:sota}
Having demonstrated that certain model compression techniques are capable of producing \cnets, we explore using existing techniques for improving model robustness in conjunction with compression strategies to produce \cnets that further improve robustness.
We consider three popular existing strategies for improving model robustness and, further, propose a test-time adaptive ensembling strategy, called a domain-adaptive \cdeck, that leverages these strategies to maintain compactness and efficiency while improving accuracy and robustness over individual \cnets.

\subsection{Popular strategies for improving model robustness}
\paragraph{Data augmentation.}
A popular approach for improving robustness involves data augmentations.
We consider two augmentation techniques that are (at the time of writing) leading methods on RobustBench \cite{croce2020robustbench}.  
The first is AugMix \cite{hendrycks2019augmix} which can provide improved robustness without compromising accuracy
by randomly sampling different augmentations, applying them to a training image, then ``mixing" the resulting augmented image with the original.
The second method independently adds Gaussian noise $\mathcal{N}(\mu = 0,\sigma =0.1)$
to all the pixels with probability $p =0.5$ \cite{kireev2021effectiveness}.

\paragraph{Larger models.}
Another popular strategy for improving OOD robustness (and accuracy) in the robust DL community is to increase the model size (e.g., \citep{hendrycks2020many, gowal2020uncovering,croce2020robustbench}).
Hence, we also consider this strategy to investigate if the performance of \cnets can be amplified by compressing larger models. 

\paragraph{Model Ensembling.} It is natural to consider exploiting \cnet compactness to amplify accuracy and robustness by ensembling \cite{polikar2012ensemble}  \cnets. 
For example, an ensemble of two to six \cnets pruned to 95\% sparsity only uses 10\% to 30\% of the parameter count required by a single dense model. 


\subsection{Playing the right CARD to improve accuracy-robustness performance}
Ensembling \cnets trained with state-of-the-art data augmentation techniques has the potential to provide additional robustness gains.
We call such ensembles \cdecks and propose two strategies: (1) domain-agnostic \cdecks and (2) domain-adaptive \cdecks. 
In both strategies, the ensemble consists of \cnets that have been trained on the same dataset under different augmentation schemes. 
The domain-adaptive \cdeck utilizes a spectral-similarity metric to select a subset of \cnets from the \cdeck that should be used to make predictions based on the current test data.
We first define this metric then provide formal definitions for both \cdeck methods.

\paragraph{A spectral-similarity metric.}
Let $x_{train} \in \mathbb{R}^{D_1 \times D_2 \times N}$ denote the N unaugmented training images of dimension $D_1 \times D_2$, $A = \{ a_k \}_{k=1}^m$ denote a set of $m$ different augmentation schemes, and $\hat{S}_{a,P}(x)$ denote a sampling of $P$ images from $x$ where augmentation $a \in A$ has been applied to $x$. Motivated by our analysis using Fourier heatmaps, 
we propose a spectral-similarity metric to compare representatives from augmented versions of the training sets, $\{ \hat{S}_{a,P}(x_{train})\}_{a \in A}$, to the test data. First, we define $F(\cdot)$ 
as a map that computes the 1D radially-averaged power spectrum for images of dimension $D_1 \times D_2$ then takes the reciprocal of each component. 
Our spectral-similarity metric is a map $d_{ss}: \mathbb{R}^{D_1 \times D_2 \times P} \times \mathbb{R}^{D_1 \times D_2 \times M} \to \mathbb{R}$ defined by $d_{ss}(\bm{X}, \bm{Y}) = \min_{1 \leq i \leq P} \| (F(X_i) / \| F(X_i) \|) - \frac{1}{M} \sum_{j=1}^{M} (F(Y_j) / \| F(Y_j) \|) \|$. 
In practice, we found that the 1D power spectra for different augmentation types were more separable in the higher frequencies of the power spectrum leading to the use of the reciprocal in the definition of $F(\cdot)$. 


\begin{figure}[ht]
    \centering
    \includegraphics[width=\textwidth]{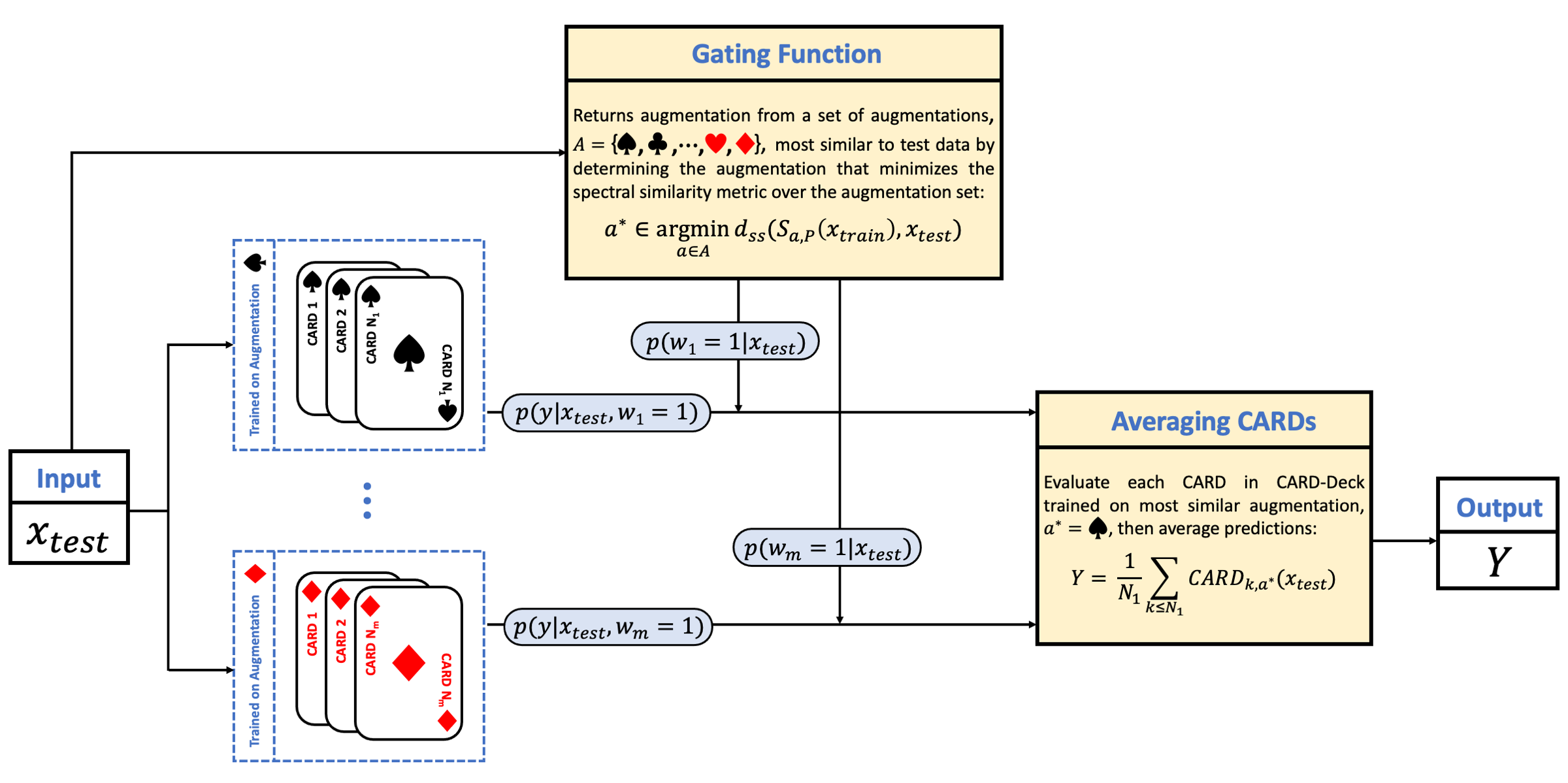}
    \caption{{\bfseries Selecting a ``winning hand" using a domain-adaptive \cdeck}: \cnets are grouped based on the data augmentation scheme used during training. 
    At test time, a gating function 
    identifies the augmentation scheme, $a^*$ with data most similar to the test data. 
    The \cnets trained using $a^*$ are used and the average prediction is returned as the output.}
    \label{fig:card-deck-visualization}
\end{figure}

\paragraph{A ``winning hand'' of CARDs by test-time ensembling.}
An $n$-\cdeck ensemble is composed of $n$ \cnets given by $f^{Deck} = \{ f^{a_{k}} \}_{k=1}^n$ where $a_k$ is one of the $m$ augmentation schemes from $A$ and the superscript in $f^{a_{k}}$ denotes that this \cnet was trained used data from the distribution $S_{a_k} (x_{train})$. 
Our domain-agnostic $n$-\cdeck averages the prediction of all $n$ \cnets in the deck. Supposing that the output of each \cnet in $f^{Deck}$ is softmax vectors, then the output of the domain-agnostic $n$-\cdeck can be expressed as $\frac{1}{n} \sum_{k=1}^n f^{a_k}(x_{test})$. 
In our domain-adaptive \cdeck, a gating module uses the spectral-similarity metric $d_{ss}$ to determine which augmentation method is most similar to a batch of $M$ test images $x_{test} \in \mathbb{R}^{D_1 \times D_2 \times M}$ provided to the ensemble.
When an augmentation scheme, say $a \in A$, is identified as the most similar to the incoming test data, the domain-adaptive \cdeck utilizes only the \cnets that were trained using the data from the distribution $S_{a}(x_{train})$. 
The set of the most similar augmentations is given by $a^* \in \arg\min_{a \in A} d_{ss} (\hat{S}_{a,P} (x_{train}), x_{test})$. 
We note that $a^*$ is likely to be a singleton set indicating that a single augmentation scheme is most similar. 
If the domain-adaptive \cdeck contains multiple \cnets trained using the same data augmentation scheme, prediction averaging is used on these \cnets and returned as the \cdeck prediction. 
Given $a^*$ and letting $\mathcal{I}(a) = \{ k : a_k \in a\}$, the output of the domain-adaptive \cdeck can be expressed as $\frac{1}{|\mathcal{I}(a^*)|} \sum_{k \in \mathcal{I}(a^*)} f^{a_k}(x_{test})$.
As computing the spectral-similarity scheme is independent of \cnet evaluation, the domain-adaptive \cdeck provides reduced inference time over the domain-agnostic \cdeck by only evaluating the \cnets necessary for prediction. 
Figure~\ref{fig:card-deck-visualization} provides an illustration of the \cdeck design.


\subsection{Experimental results}

We experiment with four models of increasing size
(ResNeXt-29, ResNet-18, ResNet-50, WideResNet-18-2), three data augmentation methods (clean, AugMix, Gaussian), two sparsity levels (90\%, 95\%), and six compression methods (LTH, LRR, EP (layerwise and global), BP (layerwise and global)). 
For each model, sparsity level, data augmentaion method, and compression method, three realizations are trained on CIFAR-10~\cite{krizhevsky2009learning} and robustness is measured using CIFAR-10-C. Model compactness is measured by calculating the memory usage~\citep{xu2019main}. Similar experiments are performed for CIFAR-100 and CIFAR-100-C, however only 
WideResNet-18-2 and four model compression methods (LTH, LRR, EP (global), BP (global)) are used. As a baseline, three realizations of each model are trained without compression for each data augmentation method. Visualizations of key results are provided in this section and detailed ablation studies are in Appendix~\ref{appendix:card-deck}.

In addition to measuring the performance of \cnets for each configuration (i.e. model, data augmentation, compression method, sparsity level), we also formed domain-agnostic and domain-adaptive $n$-\cdecks of size $n \in \{2, 4, 6\}$ comprised of models using the same compression method and sparsity level. 
For each $n$-\cdeck, half of the \cnets were trained using AugMix and the other half were trained using the Gaussian augmentation. 
To facilitate computation of the spectral-similarity metric in domain-adaptive \cdecks, for each augmentation method $a \in \{ AugMix, Gaussian\}$ we statically created a KD-Tree containing $F(X)/\| F(X) \|$, for all $X \in \hat{S}_{a,P}(x_{train})$.
In our experiments, we took $P = 5000$ and these KD-Trees were generated once and saved (separate from inference process). 
At test time, batches of $M=100$ test images were used in the spectral-similarity metric to determine which augmentation method best represented the corrupted test data. 

\paragraph{Test-time ensembling can provide a ``winning hand''.}
\begin{figure}
\centering
\begin{subfigure}{.5\textwidth}
  \centering
  \includegraphics[width=0.95\textwidth]{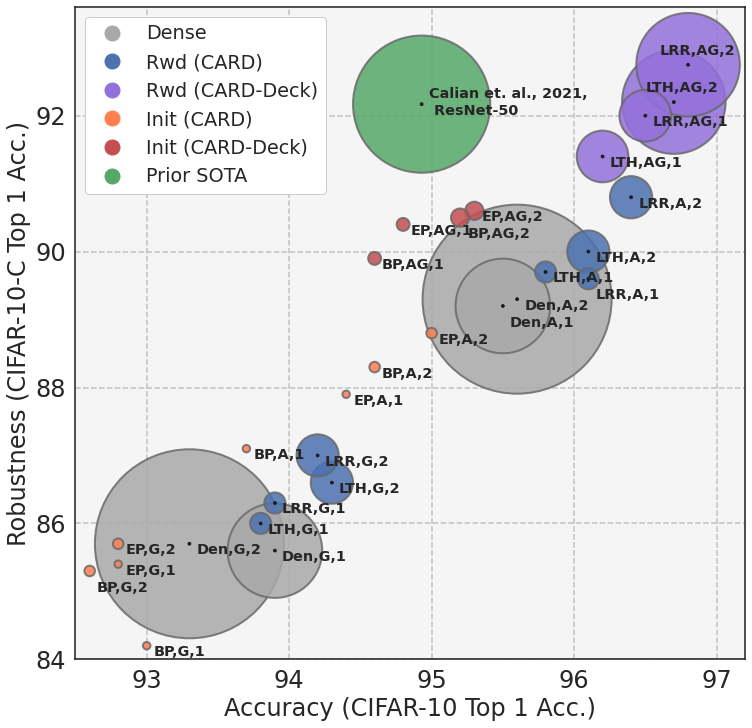}
  \caption{CIFAR-10-C}
  \label{fig:c10-arm}
\end{subfigure}%
\begin{subfigure}{.5\textwidth}
  \centering
  \includegraphics[width=0.95\textwidth]{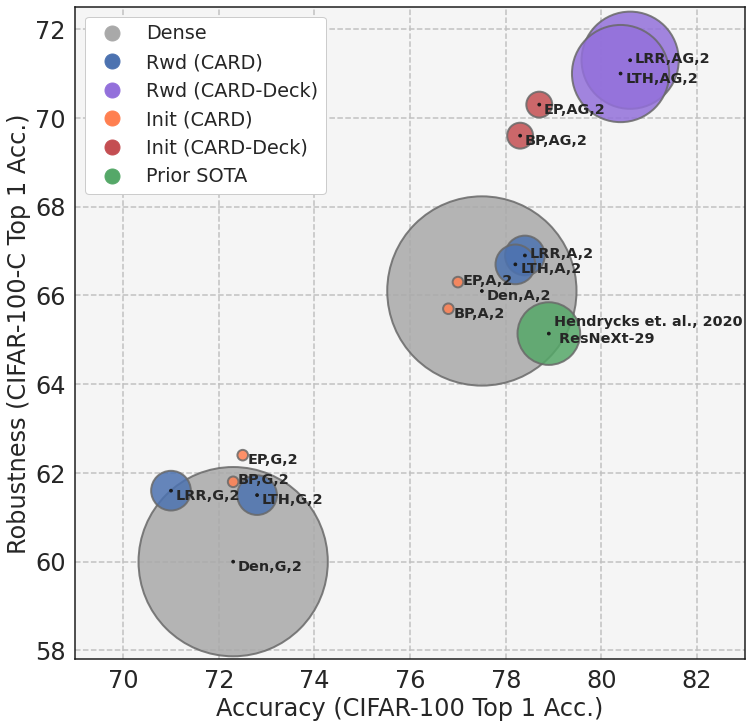}
  \caption{CIFAR-100-C}
  \label{fig:c100-arm}
\end{subfigure}
\caption{\textbf{Accuracy, robustness, and memory trends}. \cnets and \cdecks reduce memory usage (as indicated by the area of each circle) while achieving comparable or improved accuracy (x-axis) and robustness (y-axis).
Annotation indicates \{Compression Method, Data Augmentation, ResNet-18 Width\} with \textbf{A} for AugMix, \textbf{G} for Gaussian, and \textbf{AG} for both (used by \cdecks).
}
\label{fig:acc-rob-mem}
\end{figure}
Figure~\ref{fig:acc-rob-mem} provides a visualization of the performance (accuracy, robustness, and memory usage) of several \cnets and \cdecks as well as dense baselines and the previous SOTA model.
This figure highlights our findings that both \cdeck methods, domain-agnostic and adaptive, are capable of improving the performance beyond the dense baselines while maintaining reduced memory usage.
Notably, we found a single LRR \cnet (a WideResNet-18 at 96\% sparsity) trained with AugMix can attain 91.24\% CIFAR-10-C accuracy, outperforming dense ResNeXt-29 trained with AugMix (a state-of-the-art among methods that do not require non-CIFAR-10 training data) by more than 2 percentage points simply by pruning a larger model, i.e., WideResNet-18.
Our best performing 6-\cdeck using LRR WideResNet-18 models (53.58 MB) sets a new state-of-the-art for CIFAR-10 and CIFAR-10-C accuracies of 96.8\% and 92.75\%, respectively. 
In contrast, the previous best method~\citep{calian2021defending} achieves accuracies (94.93\%, 92.17\%) using increased memory (ResNet-50 with 94.12 MB), extra data (a super resolution network was pre-trained with non-CIFAR-10 data), and a computationally expensive adversarial training procedure. 
More impressively, our computationally lighter binary-weight \cdecks provide
comparable accuracy and robustness to the dense baseline with memory usage as low as 1.67 MB. 
Similar trends hold on CIFAR-100-C where \rwd domain-adaptive \cdecks set a new SOTA performance (80.6\%, 71.3\%) compared to the previous best (78.90\%, 65.14\%) \cite{hendrycks2019augmix}. Note that the binary-weight \cdecks provide almost 5 percentage point robustness gain over the previous best at only 9\% of the memory usage.
Note that the performance of EP and BP \cdecks can be further improved by leveraging more computationally expensive training procedures, e.g., tuning batchnorm parameters \cite{diffenderfer2021multiprize} or searching for EP and BP \cnets in pretrained neural nets.

To summarize, 
\cdecks can maintain compactness while leveraging additional robustness improvement techniques,
LRR \cdecks 
set a new SOTA on CIFAR-10-C and CIFAR-100-C in terms of accuracy and robustness,
binary-weight \cdecks can provide up to $\sim$105x reduction in memory while providing comparable accuracy and robustness,
and the domain-adaptive \cdecks used 
here are $\sim$2x faster than the domain-agnostic \cdecks as only half of the \cnets are used at inference.
Additionally, for 2-\cdecks our domain-adaptive method provides an average robustness gain of 1-2 percentage points over the domain-agnostic method (see Appendix~\ref{appendix:card-deck}).

\section{Theoretical justifications} \label{sec:theory}
This section provides (1) theoretical results that provide support for the \cnet hypothesis beyond what we demonstrated empirically and (2) robustness analysis for domain-adaptive \cdeck strategy.

\subsection{Function approximation view of CARDs}

By leveraging existing theoretical analyses of the Strong and Multi-Prize Lottery Ticket Hypotheses \citep{pensia2021optimal, orseau2020logarithmic, diffenderfer2021multiprize}, we can provide theoretical support for the \cnet hypothesis. 
While we were able to empirically produce \cnets within the same architecture used by the dense model, to prove theoretical results supporting the \cnet hypothesis using existing techniques requires that the compressed network be searched for within an architecture larger than the architecture used for the dense model.
An informal version of this result for binary-weight \cnets is provided here relevant to models produced by joint pruning and binarization compression strategies such as multi-prize tickets \cite{diffenderfer2021multiprize}. 

\begin{Thm}
\label{thm:informal-card}
Given a non-compressed network $F$ with depth $\ell$ and width $w$ with bounded weights that achieves a desired target accuracy and robustness, a random binary network of depth $2\ell$ and width $O\left( (\ell w^{3/2}/\varepsilon) + \ell w \log(\ell w / \delta) \right)$ contains with probability $(1-\delta)$ a binary-weighted \cnet that approximates the target non-compressed network with error at most $\varepsilon$, for any $\varepsilon, \delta > 0$. 
\end{Thm}

We note that Theorem~\ref{thm:informal-card} follows immediately from Theorem 2 in \citep{diffenderfer2021multiprize}
and, thereby, refer the reader to Theorem 2 in \citep{diffenderfer2021multiprize} for a formal statement.\footnote{Following the acceptance of this paper, improved bounds on the depth and width have been established \cite{sreenivasan2021finding}.}
This result provides a level of confidence with which one can expect to find a binary-weight \cnet that is an $\varepsilon$-approximation of a target (i.e. trained and non-compressed) network.
For full-precision weight \cnets, tighter bounds on the depth and width of a fully-connected network with ReLU activations containing a \cnet that is an $\varepsilon$-approximation of a target network follow from Theorem 3 in \citep{orseau2020logarithmic} which also utilizes a more relaxed hypothesis set. 
Hence, theoretical results supporting the existence of both full-precision and binary-weight \cnets 
, with high probability, provided that 
a sufficiently overparameterized network is used. 

Leveraging these theoretical results,
we provide a corollary on the approximation capabilities of \cdeck. 
We denote by $F(\ell, \bm{w})$ a fully-connected neural network with ReLU activations where $\ell$ denotes the depth of the network and $\bm{w} = [w_0, w_1, \ldots, w_{\ell}] \in \mathbb{N}^{\ell+1}$ is a vector where component $i \in \{ 1, \ldots, \ell \}$ denotes the width of layer $i$ in $F$ and $w_0$ denotes the input dimension of $F$. 

\begin{Coro}[\cdeck Approximation Theorem] \label{thm:informal-card-deck}
Let $\varepsilon > 0$, $\delta > 0$, $n \geq 1$, and $\bm{\lambda} = [\lambda_1, \ldots, \lambda_n]$ satisfying $\sum_{k=1}^n \lambda_k = 1$ and $\lambda_k \geq 0$, for all $k \in [n]$, be given. 
Let $\mathcal{F} = \{ F_k(\ell_k, \bm{w}_k) \}_{k=1}^n$ be a deck of non-compressed fully-connected networks with ReLU activations. 
If the input space $\mathcal{X}$ and each network in the collection $\mathcal{F}$ satisfies the hypotheses of Theorem 3 in \citep{orseau2020logarithmic} (Theorem 2 in \citep{diffenderfer2021multiprize}), then with probability $(1 - \delta)^n$ there exists a deck of $n$ full-precision (binary-weight) \cnets denoted $f^{Deck} = \{ f_k \}_{k=1}^n$ of depth and width specified by Theorem 3 in \citep{orseau2020logarithmic} (Theorem 2 in \citep{diffenderfer2021multiprize}) such that
\vspace{-0.05in}
\begin{small}
\begin{align}
    \sup_{x \in \mathcal{X}} \Big\| \sum_{k=1}^n \lambda_k f_k (x) - \sum_{k=1}^n \lambda_k F_k(\ell, \bm{w}) (x) \Big\| \leq \varepsilon. \label{result:card-deck-1}
\end{align}
\end{small}
\end{Coro}

A proof of Corollary~\ref{thm:informal-card-deck} is provided in Appendix~\ref{appendix:theory}. 
Note that the target non-compressed networks in Corollary~\ref{thm:informal-card-deck} could be trained on data sampled from augmented distributions, such as augmented distributions using the AugMix and Gaussian methods, provided that the weights of the resulting networks satisfy the hypothesis required from the existing results in \citep{orseau2020logarithmic,diffenderfer2021multiprize}. 
Additionally, the appropriate choice of $\bm{\lambda}$ in Corollary~\ref{thm:informal-card-deck} can yield a domain-agnostic or domain-adaptive \cdeck.

\subsection{Robustness analysis of CARD-Deck}
To provide the theoretical justification behind our \cdeck approach over a single classifier, we first define a robustness measure for a given classifier ensemble trained on a set of augmentations w.r.t. a corruption set encountered at the test-time. 
We assume that each test sample may encounter a specific corruption type $c$ from a given OOD set and be transformed to a corrupted test sample $x_c$.
Let us assume $f^a$ is learnt using a learning algorithm $L$ using the augmented training data $S_a$ sampled from distribution $\mathcal{D}_a$, thus, we have $f^{Deck}=\{f^a = L(S_a)| a \in \mathbb{N}^A\}$ where $\mathbb{N}^A = \{1,\cdots,A\}$. Let us denote by $\hat{S_a}$ an empirical distribution w.r.t. sampled dataset $S_a$. 

\begin{Df}[Average OOD Robustness]
Let $f^{Deck} = \{f^a| a \in \mathbb{N}^A\}$ denote a \cdeck trained using an augmentation set ${S^A = \{S_a \sim \mathcal{D}_a| a \in \mathbb{N}^A\}}$. We define the average out-of-distribution robustness for a \cdeck w.r.t. corruption set $\mathcal{D}^C = \{\mathcal{D}_c| c \in \mathbb{N}^C\}$ as
\vspace{-0.05in}
\begin{small}
\begin{equation}
\label{avg_robust}
   Rob(\mathcal{D}^C, f^{Deck}) = \sum_{a=1}^{|A|}\sum_{c=1}^{|C|} Rob(\mathcal{D}_c, f^{a}) w_c^a,
\end{equation}
\end{small}
where $Rob(\mathcal{D}_c, f^{a}) = \E_{(x_c,y)\sim \mathcal{D}_c} \left[ \inf_{f^a(x_c')\neq y} d(x_c', x_c)\right]$ with $x_c'$ being a perturbed version of $x_c$, $d$ corresponds to a distance metric, and $w_c^a$ denotes the probability of $f^{Deck}$ gating module selecting the classifier $f^a$ to make a prediction on test data coming from corruption type $c$. 
\end{Df}
This definition refers to the expectation of the distance to the closest misclassified corrupted sample for a given test sample. Note that this is a stronger notion of robustness then the generalization error corresponding to a corrupted data distribution. 
Having this definition, our goal is to provide a lower bound on the average OOD robustness of $f^{Deck}$ and show that the use of domain-adaptive classifier ensemble achieves a better OOD robustness compared to the case where we use just a single classifier $f^a$. 
To understand this quantity better, we derive the following decomposition (see Appendix~\ref{appendix:theory}):
\begin{small}
\begin{equation*}
Rob(\mathcal{D}^C, f^{Deck}) \geq \sum_{a,c} w_c^a[ \underbrace{Rob(\hat{S_a}, f^{a})}_{(a)}
         -\underbrace{\|Rob(\mathcal{D}_a, f^{a}) - Rob(\hat{S_a}, f^{a})\|}_{(b)}
         -\underbrace{\|Rob(\mathcal{D}_c, f^{a})- Rob(\mathcal{D}_a, f^{a})\|}_{(c)}].    
\end{equation*}
\end{small}
This shows that the average OOD robustness can be bounded from below in terms of the following three error terms for a classifier-corruption pair weighted by their selection probabilities: (a) empirical robustness, (b) generalization gap, and (c) out-of-distribution-shift. 
This implies that in order to bound the average OOD robustness, we need to bound both the generalization gap and the OOD-shift. Next, we provide a bound on the OOD-shift penalty that is independent of the classifier $f^a$ and is only related to the closeness of the augmented data distribution and corrupted data distribution. The closeness is defined in terms of Wasserstein distance $W(\cdot, \cdot)$ (see Definition \ref{wasserstein} in Appendix~\ref{appendix:theory}).

\begin{Thm}[Average OOD-Shift Bound]
\label{dist-shift}
For any \cdeck, the average OOD-shift (i.e., $\text{ADS} = \sum_{a,c} {w_c^a \|Rob(\mathcal{D}_c, f^a) - Rob(\mathcal{D}_a, f^a) \|}$) can be bounded as follows
$ADS \leq \sum_{a=1}^{|A|}\sum_{c=1}^{|C|} {w_c^a \times W(\mathcal{D}_c, \mathcal{D}_a)}$.
\end{Thm}
\begin{proof}
This result can be proved by applying Theorem 1 in \citep{sehwag2021improving} to ADS. 
\end{proof}

\noindent\textbf{Key insights.} Theorem \ref{dist-shift} provides some key insights into the OOD robustness of classifiers trained on augmented datasets. First, unlike the generalization gap, the OOD-shift does not converge to zero with more augmentation data. This imposes a fundamental limit on the OOD robustness in terms of the distance between augmented train data distribution and corrupted test data distribution.
Second, having diverse augmentations is critical to improving the OOD robustness. Also, it highlights that existing solutions trained with a single augmentation scheme might just be getting lucky or overfitting to the corrupted test data.
Finally, the domain-adaptive $\cdeck$ with a suitable gating function is provably better than using a single classifier because it can achieve the minimum conditional Wasserstein distance (or best achievable OOD robustness) over given augmentation-corruption pairs.

\section{Limitations and future directions}
\label{sec:limitation}
In this paper, we showed that model compression and high robustness (and accuracy) are not necessarily conflicting objectives. 
We found that compression, if done properly (e.g., using ``lottery ticket-style'' objectives), can improve the OOD robustness compared to a non-compressed model. 
Leveraging this finding, we proposed a simple domain-adaptive ensemble of \cnets that outperformed existing SOTA in terms of the clean accuracy and the OOD robustness (at a fraction of the original memory usage).   
Our results are consistent with past results in that we also show that the use of test accuracy alone to evaluate the quality/deployability of a compressed model in the wild is not sufficient---one needs to adopt harder metrics such as OOD robustness. However, as opposed to the existing works in this direction, we present a construction that satisfies these ``harder'' requirements. 

There are still many interesting questions that remain to be explored. 
First, while we were able to produce \cnets it remains unclear (i) why only certain pruning strategies were able to produce them and (ii) why introducing compression can improve ``effective robustness'' \citep{taori2019robustness} (e.g. Conv and VGG19 BP and EP models in Figure~\ref{fig:card-net-all-plot}).
Second, the spectral relationship of train and test data (as considered in this work) is not the only interaction determining the performance of a compressed model. 
It will be worthwhile to take a more holistic approach that also takes spectral behavior of the compressed model (e.g., using intermediate features) into account, which could possibly benefit from using \cnets compressed via different strategies when building a ``winning hand''. 
Third, we only derived an upper bound on the  amount of overparameterization needed to approximate a target dense network in our theoretical analysis; it will also be interesting to explore a lower bound (a necessary condition) on the same which may indicate scenarios where the proposed approach will not work (e.g., underparameterized NNs). 
Fourth, ``lottery ticket-style'' models in theory can be found more efficiently, which was not our focus but is a valuable future direction. 
Finally, achieving the theoretical memory savings obtained from \cnets (reported in this paper) would require their implementation on specialized hardware. 
We hope that our results will help researchers better understand the limits of compressed neural nets, and motivate future work on \cnets and their applications to areas where DL struggles currently due to its parameter-inefficiency and OOD brittleness.


\newpage

\section*{Acknowledgements} 
We would like to thank the reviewers for their
valuable discussion during the rebuttal period that resulted in improved clarity and presentation of our research.
This work was performed under the auspices of the U.S. Department of Energy by the Lawrence Livermore National Laboratory under Contract No. DE-AC52-07NA27344 and LLNL LDRD Program Project No. 20-ER-014 (LLNL-CONF-823802).

\bibliography{winning_hand}
\bibliographystyle{plainnat} 

\newpage
\renewcommand\appendixpagename{\Large Supplementary Material: {\em A Winning Hand}: Compressing Deep Networks Can Improve Out-Of-Distribution Robustness}

\appendix
\appendixpage

\title{The Tri}
\maketitle

Here we provide a brief outline of the appendices. 
In Appendix \ref{sec:background}, we provide details on relevant past works. 
In Appendix \ref{appendix:experiment-details}, we discuss our experimental setting and relevant hyperparameters. 
In Appendix \ref{appendix:init-experiments}, we provide additional experiments with \init methods and, in part, show that the robustness of the EP method is not only due to binarization but also due to the specific pruning strategy. 
In Appendix \ref{appendix:all_heatmaps}, we provide Fourier heatmaps for additional pruning rates and architectures. 
In Appendix \ref{appendix:n_shot}, we provide additional Fourier heatmap results on comparing the rewinding-based schemes with the traditional pruning schemes. 
In Appendix \ref{appendix:card-deck}, we provide extensive tables for \cnet and \cdeck experiments performed in Section~\ref{sec:sota}. 
In Appendix \ref{appendix:theory}, we provide remaining proof details for our theoretical justification of our \cdeck approach. We show the universal approximation power of \cdecks and prove that \cdeck with a suitable gating function is provably better than using a single classifier.

\section{Background}
\label{sec:background}

\subsection{Accuracy preserving model compression}
\label{sec:compression}
{Two popular approaches for model compression are: pruning and quantization. Here, we discuss these approaches and their effects on accuracy.}

\noindent\textbf{Pruning.}
Neural network pruning removes weights \citep{lecun1990optimal} or larger structures like filters \citep{li2016pruning} from neural networks to reduce their computational burden \citep{han2015learning,he2017channel} and potentially improve their generalization \citep{wen2016learning,louizos2017learning}. As the performance of DNNs has continued to improve with increasing levels of overparameterization \citep{zhang2016understanding}, production DNNs have grown larger \citep{krizhevsky2012imagenet,brown2020language}, and the need to broadly deploy such models has amplified the importance of compression methods like pruning \citep{han2015learning,han2015deep}.

In modern networks, pruning the smallest magnitude weights after training then fine-tuning (FT) to recover accuracy lost from the pruning event is surprisingly effective; when the pruning is done iteratively rather than all at once, this approach enables a 9x compression ratio without loss of accuracy \citep{han2015learning}. Gradual magnitude pruning (GMP) performs such iterative pruning throughout training rather than after training \citep{narang2017exploring, zhu2017prune}, recovering accuracy lost from pruning events as training proceeds, and matches or exceeds the performance of more complex methods \citep{gale2019state}.

Another form of magnitude pruning stems from work on the lottery ticket hypothesis (LTH), which posits that the final, sparse subnetwork discovered by training then pruning can be rewound to its state at initialization \citep{frankle2018lottery} or early in training \citep{frankle2020linear}, then trained in isolation to be comparably accurate to the trained dense network.  The associated pruning approach that iteratively trains the network, rewinds the weights (and learning rate schedule) to their values early in training, then trains the subnetwork is referred to here as LTH. A simpler version of this algorithm, learning rate rewinding (LRR) \citep{renda2020comparing}, only rewinds the learning rate schedule (not the weights) and achieves a state-of-the-art accuracy-efficiency frontier while being less complex than other competitive approaches  \citep{zhu2017prune,molchanov2017variational,frankle2018lottery,he2018amc}. LRR has been shown to offer small improvements to accuracy with not-too-high compression ratio  \citep{renda2020comparing}. The authors in \citep{venkatesh2020calibrate} proposed calibration mechanisms to find more effective lottery tickets.

Building on the lottery ticket hypothesis, the edgepopup (EP) algorithm introduced a way to find sparse subnetworks at initialization that achieve good performance without any further training \citep{ramanujan2019whats}. \citet{diffenderfer2021multiprize} introduced a similar pruning approach, biprop (BP),  which also performs weight binarization.

\noindent\textbf{Binarization.}
Typical post-training schemes have not been successful in binarizing pretrained models with or without retraining to achieve reasonable accuracy. Most existing post-training works~\citep{han2015deep, zhou2017incremental} are limited to ternary weight quantization.
To overcome this limitation, there have been several efforts to improve the performance of binary neural network (BNN) training. This is challenging due to the discontinuities introduced by the binarization, which makes back-propogation difficult. 
Binaryconnect~\citep{courbariaux2015binaryconnect} first showed how to train networks with binary weights within the familiar back-propagation paradigm. Unfortunately, this early scheme resulted in a significant drop in accuracy compared to its full precision counterparts. To improve performance, XNOR-Net~\citep{rastegari2016xnornet} proposed adding a real-valued channel-wise scaling factor to improve capacity. Dorefa-Net~\citep{zhou2016dorefa} extended XNOR-Net to accelerate the training process via quantized gradients.
ABC-Net~\citep{lin2017towards} improved performance by using more weight bases and activation bases at the cost of increased memory.

Notably, one can exploit the complementary nature of pruning and binarization to combine their strengths.
For example, \citet{diffenderfer2021multiprize} produced an algorithm for finding multi-prize lottery tickets (MPTs): sparse, binary subnetworks present at initialization that don't require training.  

\paragraph{Pruning algorithm framework.} The following pruning algorithm framework, inspired by those in \citep{renda2020comparing,wang2021emerging}, covers traditional-through-emerging pruning methodologies. Specifically, we define the trained subnetwork created by one pruning-retraining cycle (i.e., one pruning iteration) as:
\begin{equation}
\label{prune-method}
W_{\text{sparse}} = F_1(W_{k_1}; \mathcal{D}) \odot F_2(W_i; \mathcal{D}, \mathcal{M}, k_2),
\end{equation}
where 
$\mathcal{D}$ denotes the training dataset, 
$W_i$ denotes the weight vector at the start of training iteration $i$,\footnote{During training, $i < T$ for most pruning approaches, where $T$ is the default number of training iterations. However, fine-tuning trains for an additional set of iterations after pruning takes place at iteration $T$. Additionally, rewinding-based lottery ticket approaches (when accounting for training done by $F_2$) use $(n+1)T-nr$ training iterations, where $n$ is the number of pruning iterations or ``shots'' in an n-shot pruning procedure, and $r$ is the iteration rewound to after each pruning iteration (note that when $r=0$, the network is rewound to its state from initialization after each pruning iteration and $(n+1)T$ total iterations are required by this approach).}
$F_1$ represents the function that finds and returns the weight-masking vector $\mathcal{M}$,
$F_2$ represents the function that retrains the weights after $\mathcal{M}$ is found,
$k_i$ is the earliest training iteration that $F_i$ requires information from (e.g., weight-vector or learning-rate values),
$F_1$ and $F_2$ are each applied at the beginning of iteration $k_1$,
and $\odot$ is the Hadamard (element-wise) product. 
Using this, the pruning paradigms and representative techniques from these categories considered in this paper are as follows: 
\begin{itemize}
    \item \textbf{Traditional} $k_1=k_2$ and $F_2 \neq I$ (identity function). Particular techniques:
        \begin{itemize}
            \item Fine-Tuning (FT) \citep{han2015learning}: \\ 
            \null \quad  $W_{\text{sparse}} = F_1(W_T) \odot F_2(W_T;                      \mathcal{D}, \mathcal{M}, T)$
            \item Gradual Magnitude Pruning (GMP) \citep{zhu2017prune}: \\ 
            \null \quad  $W_{\text{sparse}} = F_1(W_t) \odot F_2(W_t;                       \mathcal{D}, \mathcal{M}, t)$, where $t \in [t_1, t_2, ..., t_n]$ and $t_n<T$
        \end{itemize}
    \item \textbf{Rewinding-based Lottery Ticket} $k_1 = aT-(a-1)r$ and $k_2 = r$, where $a \in [1\ ..\ n]$ and $r\ll T$. Particular techniques:
        \begin{itemize}
            \item Weight Rewinding (LTH) \citep{frankle2018lottery,frankle2020linear}: \\ 
            \null \quad $W_{\text{sparse}} = F_1(W_{aT-(a-1)r}) \odot F_2(W_r;                     \mathcal{D}, \mathcal{M}, r)$
            \item Learning Rate Rewinding (LRR) \citep{renda2020comparing}: \\ 
            \null \quad  $W_{\text{sparse}} = F_1(W_{aT-(a-1)r}) \odot F_2(W_{aT-(a-1)r};                     \mathcal{D}, \mathcal{M}, r)$
        \end{itemize}
    \item \textbf{Initialization-based (Strong) Lottery Ticket} $k_1=k_2=0$ and $F_2 =I$. Particular techniques:
        \begin{itemize}
            \item Edgepopup (EP) \citep{ramanujan2019whats}:  \\ 
            \null \quad             $W_{\text{sparse}} = F_1(W_{0,\;                                       \text{binary}}; \mathcal{D}) \odot I(W_{0,\;                            \text{binary}})$
            \item Biprop (BP) \citep{diffenderfer2021multiprize}: \\ 
            \null \quad  $W_{\text{sparse}} = F_1(W_0; \mathcal{D}) \odot I( W_{0,\;                            \text{binarized by biprop}})$
        \end{itemize}
\end{itemize}

Note that GMP, LTH, and LRR are all iterative. Further, since rewinding schemes apply $F_1$ and $F_2$ at the beginning of iterations $k_1=aT-(a-1)r$, $a \in [1\ ..\ n]$, it's true that $k_1 > k_2 = r$, so $F_2$ needs to store information from iteration $k_2=r$ in order to (at $k_1$) perform the training iterations that determine $W_i,\ i\geq T$.
As opposed to traditional and rewinding schemes, strong lottery ticket \citep{ramanujan2019whats} schemes do not require any weight training before or after pruning---a performant network is found at initialization via $F_1$. In other words, learning occurs simply by pruning a randomly initialized neural network. 
Furthermore, by design BP performs binarization of the weights to reduce the memory footprint. We note that the precision of the weights in networks trained using EP maintain the same precision as the randomly initialized weights. Hence, EP can also be used to identify binarized networks by randomly initializing the weights to binary values. To take advantage of additional compression, in our experiments with EP the mask $\mathcal{M}$ is learned from a binary-initialized weight vector $W_{0,\;\text{binary}}$. As BP performs binarization during pruning, a full-precision weight vector $W_0$ is used when finding $\mathcal{M}$.
In all of these methods, we make use of global unstructured pruning which allows for different pruning percentages at each layer of the network. 

\subsection{Accuracy preserving robust training}
\label{sec:A-R}
While DNN models show impressive generalization in I.I.D. data scenarios~\citep{tan2021efficientnetv2,foret2020sharpness}, the robustness of such models on OOD data (e.g., common corruptions -- blurring from camera movement, or noise from low-lighting conditions) is critical to the successful deployment of DL in the wild. To evaluate  performance in the presence of such common corruptions, \citet{hendrycks2019benchmarking} introduced the CIFAR-10-C dataset, which comprises validation images from CIFAR-10 \citep{krizhevsky2009learning} that were exposed to $15$ diverse corruption types applied at $5$ severity levels.  

To achieve high OOD robustness and accuracy, AugMix \citep{hendrycks2019augmix} creates data augmentations at training time by composing randomly-selected augmentation operations from a diverse set, which notably excludes augmentations overlapping with those used to create CIFAR-10-C. Additionally, AugMix utilizes a Jensen-Shannon Divergence consistency loss term to match the predictions between different augmentations of a given image.
This approach is expanded on by DeepAugment~\citep{hendrycks2020many}, which inputs clean images to a pretrained image-to-image model, corrupts this model's weights and activations with various operations that distort the typical forward pass, then uses the output images as augmented data. AdversarialAugment (AdA) builds on DeepAugment by generating the weight perturbations performed on the image-to-image models via adversarial training \cite{calian2021defending}. Also, when used with an appropriately selected perturbation radius and distance metric, adversarial training can serve as a strong baseline against common corruptions \citep{gowal2020uncovering, kireev2021effectiveness}.

Notably, the state-of-the-art in OOD robustness has historically evolved by leveraging more advanced data augmentation schemes and larger models than prior works~\cite{croce2020robustbench}. 


\subsection{Methods to design compact-accurate-robust models}
\label{previous_attempts}
Despite its critical need, efforts towards achieving model compactness, high accuracy, and OOD (natural corruption) robustness simultaneously have mostly been unsuccessful, to the best of our knowledge. 
Note that some recent works have shown successful attempts for different use cases, e.g., adversarial example robustness~\cite{wang2020achieving}, additive white noise robustness~\cite{ahmad2019can}, and domain generalization~\citep{zhang2021can}.

\citet{hooker2019compressed} analyzed traditional compression techniques \citep{zhu2017prune} and showed that pruned and quantized models have comparable accuracy to the original dense network \textit{but} are far more brittle than non-compressed models in response to small distributional changes that humans are robust to. It is well known that even non-compressed models are very brittle to the OOD shifts. The authors in \citep{hooker2019compressed} showed that this brittleness is amplified at higher levels of compression. 

\citet{liebenwein2021lost} corroborated that a pruned \citep{renda2020comparing,baykal2019sipping} model can have similar predictive power to the original one when it comes to test accuracy, while being more brittle when faced with out of distribution data points. They further showed that this phenomenon holds even when considering robust training objectives (e.g., data augmentation). Their results suggest that robustness advances discussed in Sec.~\ref{sec:A-R} may be suboptimal with model compression approaches unless OOD shifts are known at train time. 

Notably, the aforementioned papers only analyze a limited class of pruning approaches. Our findings with traditional pruning approaches are consistent with the findings of \cite{hooker2019compressed}, which involved a traditional pruning approach.
Additionally, when \citet{liebenwein2021lost} employ a lottery ticket-style pruning approach, they find pruning harms robustness more when using smaller networks, which is consistent with our \cnet hypothesis that states that the starting network must be sufficiently overparameterized.



\section{Experiment settings}
\label{appendix:experiment-details}

All codes were written in Python using Pytorch and were run on IBM Power9 CPU with 256 GB of RAM and one to two NVIDIA V100 GPUs. Publicly available code was used as the base for each pruning method for models pruned with FT and GMP\footnote{\url{https://github.com/RAIVNLab/STR}}, LTH and LRR\footnote{\url{https://github.com/facebookresearch/open_lth}}, EP\footnote{\url{https://github.com/allenai/hidden-networks}} and BP\footnote{\url{https://github.com/chrundle/biprop}}. We added functionality for global pruning in FT, GMP, EP and BP as it was not implemented in existing repositories.

ResNet-18 results for \rwd strategies, LRR and LTH, make use of regular ResNet-18 \cite{he2015deep} models while all other methods, including dense, make use of PreAct ResNet-18 \citep{he2016identity} as it provided improved performance in terms of accuracy and robustness. 

A breakdown of hyperparameters by model and pruning method is provided in Table~\ref{table:hyperparams}. As mentioned in Section~\ref{sec:experiments}, for each pruning method we used hyperparameters tuned specifically for that method. The dense Conv2/4/6/8 models used a batch size of 60, as specified in Figure 2 of the original Lottery Ticket Hypothesis paper \citep{frankle2018lottery}. 
All pruned models and the remaining dense models were trained using a batch size of 128. In the LR schedule column, \emph{Cosine} denotes cosine decay while \emph{LR160} denotes a schedule that sets the learning rate to 0.01 at epoch 80 and 0.001 at epoch 120. All models trained using SGD use a momentum of 0.9.

\begin{table}
\centering
\tiny
\setlength\tabcolsep{3 pt}
\begin{tabular}{c|ccc|cc|cc|cc|cc|cc}
\toprule
{} & \multicolumn{3}{c}{Learning Rate} & \multicolumn{2}{c}{LR Schedule} &
\multicolumn{2}{c}{Optimizer} & \multicolumn{2}{c}{Weight Decay} & \multicolumn{2}{c}{Epochs} & \multicolumn{2}{c}{Pruning Details} \\ \cmidrule(lr){2-4}\cmidrule(lr){5-6}\cmidrule(lr){7-8}\cmidrule(lr){9-10}\cmidrule(lr){11-12}\cmidrule(lr){13-14}
{} & \rotatebox{90}{Conv2} &         \rotatebox{90}{Conv4/6/8} &      \rotatebox{90}{Rest} &                      
\rotatebox{90}{Conv2/4/6/8} &         \rotatebox{90}{Rest} &      \rotatebox{90}{Conv2/4/6/8} &         \rotatebox{90}{Rest} &     
\rotatebox{90}{Conv2/4/6/8} &         \rotatebox{90}{Rest} &      
\rotatebox{90}{Conv2/4/6/8} &         \rotatebox{90}{Rest} &      
\rotatebox{90}{Conv2/4/6/8} &         \rotatebox{90}{Rest} \\
\midrule
Dense  
&   2e-4 &  3e-4 &   0.1
&   None &  LR160 
&  Adam &  SGD
&   0 &  1e-4 
&   100 &  160 
&   N/A &  N/A  \\
\midrule
FT
&   \multicolumn{2}{c}{$\leftarrow$ 0.01 $\rightarrow$} &  0.1
&   Cosine & LR160 
&  \multicolumn{2}{c|}{$\leftarrow$ SGD $\rightarrow$}
&   \multicolumn{2}{c|}{$\leftarrow$ 1e-4 $\rightarrow$}
&   \multicolumn{2}{c|}{$\leftarrow$ 200 $\rightarrow$}
&   \multicolumn{2}{c}{$\longleftarrow$ Prune at epoch 160 then fine tune 40 epochs $\longrightarrow$} \\
GMP
&   \multicolumn{2}{c}{$\leftarrow$ 0.01 $\rightarrow$} &  0.1
&   Cosine & LR160 
&  \multicolumn{2}{c|}{$\leftarrow$ SGD $\rightarrow$}
&   \multicolumn{2}{c|}{$\leftarrow$ 1e-4 $\rightarrow$}
&   \multicolumn{2}{c|}{$\leftarrow$ 160 $\rightarrow$}
&   \multicolumn{2}{c}{$\longleftarrow (s_i, t,n,\Delta t) = (0,5,105,1) \longrightarrow$}  \\
LTH
&   5e-3 &  1e-2 &  0.1
&   \multicolumn{2}{c|}{$\leftarrow$ LR160 $\rightarrow$}
&  \multicolumn{2}{c|}{$\leftarrow$ SGD $\rightarrow$}
&   \multicolumn{2}{c|}{$\leftarrow$ 1e-4 $\rightarrow$}
&   \multicolumn{2}{c|}{$\leftarrow$ 160 $\rightarrow$}
&   rewind it.: 1000, rate: 20\%  &  rewind it.: 5000, rate: 20\%   \\
LRR
&   5e-3 &  1e-2 &  0.1
&   \multicolumn{2}{c|}{$\leftarrow$ LR160 $\rightarrow$}
&  \multicolumn{2}{c|}{$\leftarrow$ SGD $\rightarrow$}
&   \multicolumn{2}{c|}{$\leftarrow$ 1e-4 $\rightarrow$}
&   \multicolumn{2}{c|}{$\leftarrow$ 160 $\rightarrow$}
&   rewind it.: 1000, rate: 20\%  &  rewind it.: 5000, rate: 20\%    \\
BP
&   \multicolumn{3}{c|}{$\longleftarrow$ 0.1 $\longrightarrow$} 
&   \multicolumn{2}{c|}{$\leftarrow$ Cosine $\rightarrow$}
&  \multicolumn{2}{c|}{$\leftarrow$ SGD $\rightarrow$}
&   \multicolumn{2}{c|}{$\leftarrow$ 1e-4 $\rightarrow$}
&   \multicolumn{2}{c|}{$\leftarrow$ 250 $\rightarrow$}
&   \multicolumn{2}{c}{$\longleftarrow$ All Epochs $\longrightarrow$} \\
EP
&   \multicolumn{3}{c|}{$\longleftarrow$ 0.1 $\longrightarrow$} 
&   \multicolumn{2}{c|}{$\leftarrow$ Cosine $\rightarrow$}
&  \multicolumn{2}{c|}{$\leftarrow$ SGD $\rightarrow$}
&   \multicolumn{2}{c|}{$\leftarrow$ 1e-4 $\rightarrow$}
&   \multicolumn{2}{c|}{$\leftarrow$ 250 $\rightarrow$}
&   \multicolumn{2}{c}{$\longleftarrow$ All Epochs $\longrightarrow$} \\
\bottomrule
\end{tabular}
\vspace{2mm}
\caption{Hyperparameters used when training dense baselines and each pruning method by model. Note that ``Rest'' refers to all other models trained in our experiments, such as VGG and ResNet.}
\label{table:hyperparams}
\end{table}

We first note details of experiments using traditional pruning methods, fine-tuning (FT) and gradual magnitude pruning (GMP). For FT models, unpruned training takes place for 160 epochs at which point pruning to the full sparsity level takes place using global magnitude pruning. After pruning, fine-tuning of the pruned network takes place over 40 epochs where the learning rate is kept at the final value after pruning at epoch 160 \citep{liu2018rethinking,renda2020comparing}.  For GMP models, the sparsity level gradually increases over the course of the training process. In our experiments, the sparsity level at training step $t$ increases in accordance with equation (1) from \citep{zhu2017prune} which we include here to interpret the GMP pruning details from Table~\ref{table:hyperparams}:
\begin{align}
    s_t = s_f + (s_i - s_f) \left( 1 - \frac{t - t_0}{n \Delta t} \right)^3, \ \text{for} \ t \in \{ t_0 + k \Delta t \}_{k=0}^n.
\end{align}
Here, $s_i$ denotes the initial sparsity level, $s_f$ denotes the final sparsity level, $n$ denotes the number of pruning steps, $t_0$ denotes the first training step where pruning is performed, and $s_t$ denotes the sparsity level at the current training step. Note that the values for $s_i$, $t_0$, $n$, and $\Delta t$ are provided in Table~\ref{table:hyperparams}. 

For rewinding methods, LTH and LRR, hyperparameters were chosen based on details from \citep{frankle2018lottery,frankle2020linear,openlth,renda2020comparing}. Notably, our rewinding-iteration choices stemmed from the hyperparameter study shown in Figure 7 of \citep{frankle2020linear}, and the fact that the small Conv models performed well when rewound to iteration 0 in \citep{frankle2018lottery}. All LTH/LRR runs were implemented using a modified version of the OpenLTH repository \citep{openlth}. 

For initialization methods, edgepopup (EP) and biprop (BP), pruning is achieved by learning a pruning mask that is applied to the randomly initialized networks weights and, in the case of BP, binarization is applied to the weights of the resulting pruned network. For EP networks, weights were initialized using the signed constant initialization from \citep{ramanujan2019whats} which offered the best performance. As an added benefit for compactness, this initialization also yields a binary weight network. For BP networks, weights were initialized using the kaiming normal initialization as in \citep{diffenderfer2021multiprize} and the biprop algorithm performs binarization during training resulting in a binary-weight network. Due to the binary weights in both the EP and BP \cnets we trained, these \cnets provided further reductions in on-device memory consumption over rewinding based pruning strategies. For both EP and BP, we used the same number of epochs for training as in \citep{diffenderfer2021multiprize}.

\section{Additional Experiments} \label{appendix:init-experiments}
\subsection{Effect of global vs. layerwise pruning in lottery ticket \init methods.} \label{appendix:layer-vs-global}
The lottery ticket \init methods analyzed in the Section~\ref{sec:experiments} were originally developed to prune a percentage of weights uniformly across all layers of the network. In contrast, global pruning methods are considered to be more flexible as they can prune some layers more heavily than others while still meeting a user-specified sparsity level for the entire network. 
By analyzing these \init methods using both layerwise and global pruning, we notice certain peculiar patterns.
Figure~\ref{fig:global-vs-layer-main-plot} provides the accuracy and robustness of models trained with BP and EP using global and layerwise pruning. 
For each model, the maximum CIFAR-10 accuracy was achieved by a layerwise pruned model at one of the six sparsity levels.
However, the globally-pruned models consistently outperform the layerwise pruned models on robustness at nearly every sparsity level.
Furthermore, the globally-pruned models typically achieve higher or comparable accuracy at higher sparsity levels, indicating that \init methods utilizing global pruning are more suitable when a high-level of sparsity is desired.

\begin{figure}[h]
    \centering
    \includegraphics[width=\textwidth,clip]{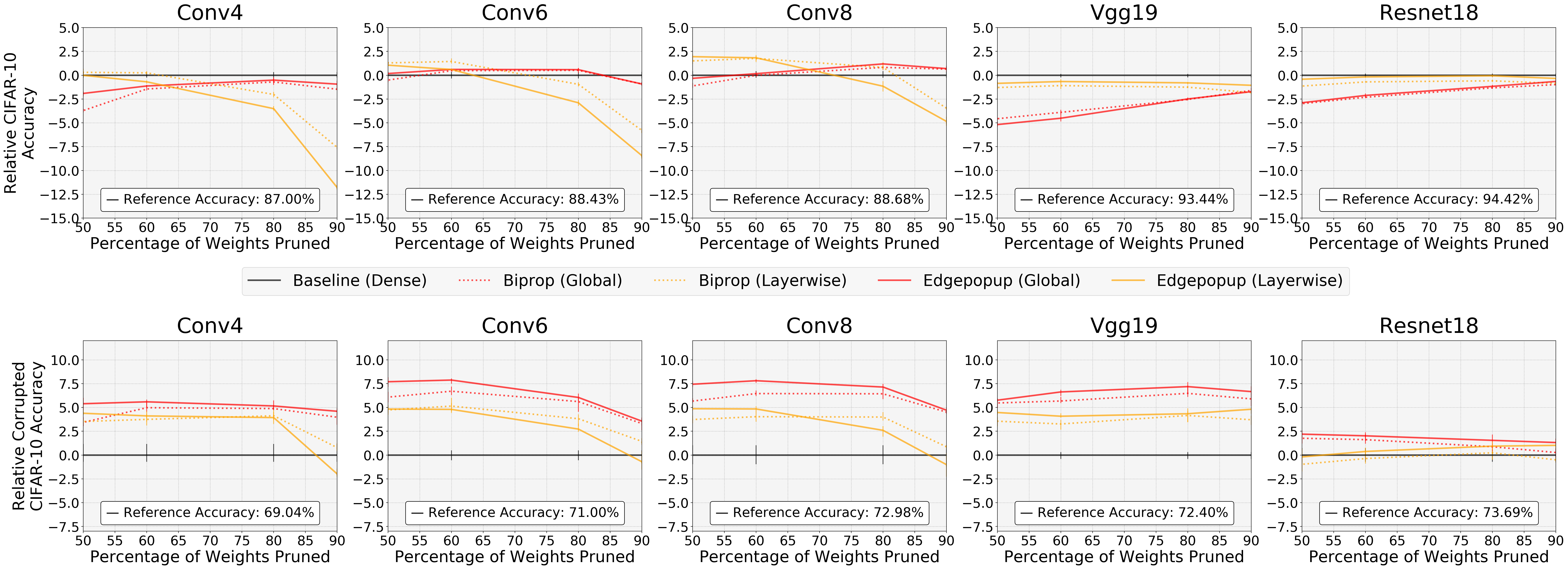}
    \caption{{\bfseries Global pruning in lottery ticket \INIT methods provides greater robustness gains}: While layerwise pruning is able to achieve the highest accuracy across all sparsity levels in \init methods, global pruning provides more significant robustness gains at all sparsity levels.}
    \label{fig:global-vs-layer-main-plot}
\end{figure}

\subsection{Comparison of full-precision-weight Edgepopup pruning with binary-weight Edgepopup pruning}
\label{appendix:EP_full_precision}
The models pruned using EP in our experiments are pruned using weights initialized from a scaled binary initialization, as specified in \cite{ramanujan2019whats}. Additionally, models pruned with BP contain binary weights regardless of the initialization used. To demonstrate that the robustness gains afforded are a feature of initialization based pruning methods and not binarization, we provide some results for full-precision initialization based pruning models. In particular, by using the kaiming normal initialization with EP the resulting network has full-precision weights. In Figure~\ref{fig:single-vs-binary-plot}, we visualize the accuracy of these models on CIFAR-10 and CIFAR-10-C. 
These experiments demonstrate the the robustness of the initialization based \cnets is not exclusive to binary weight networks as the full-precision weight networks can achieve comparable accuracy to the binary weight networks at some prune percentages. 

\begin{figure}[h]
    \centering
    \includegraphics[width=\textwidth,clip]{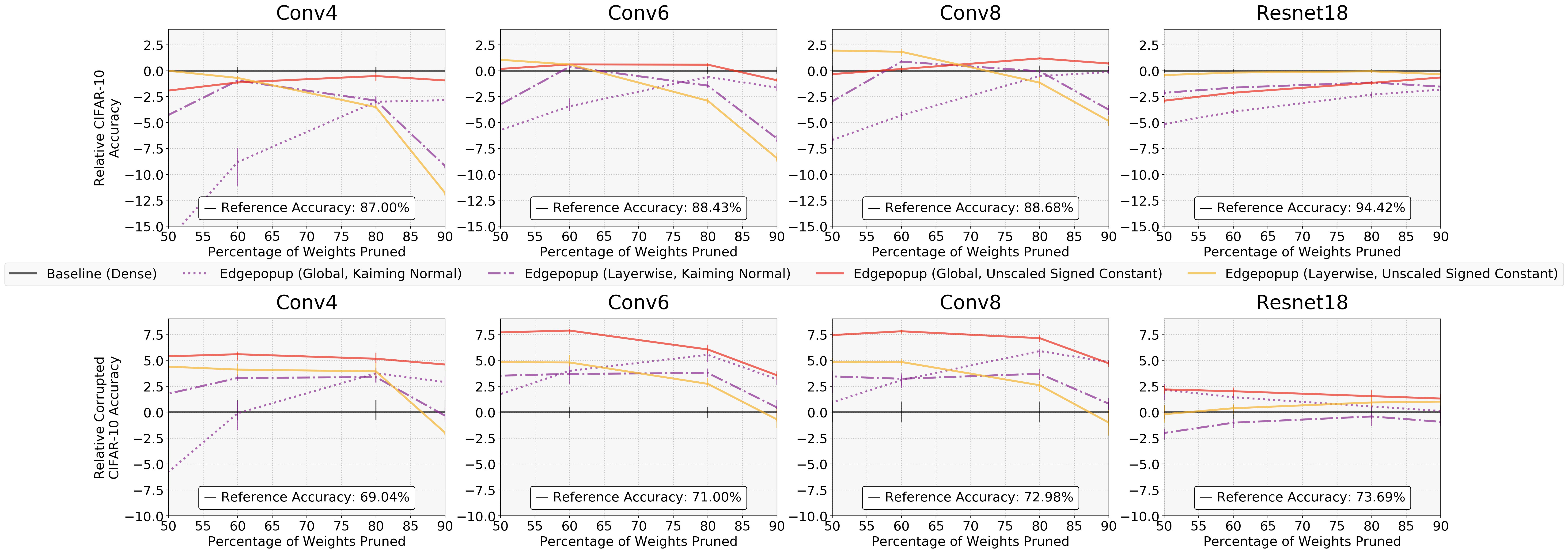}
    \caption{{\bfseries Using Full Precision weights in lottery ticket \INIT still provides robustness gains}: While initialization pruning methods with binary weights yield the greatest robustness gains over the baseline, randomly initialized networks with full-precision weights pruned using Edgepopup are capable of providing improved robustness over the dense baseline.}
    \label{fig:single-vs-binary-plot}
\end{figure}

\section{Additional heatmaps}
\label{appendix:all_heatmaps}
Here we provide additional heatmaps (varying sparsity levels) for Conv8 (see Figures~\ref{fig:fourier-conv8-90-percent} and \ref{fig:fourier-conv8-95-percent}) and for ResNet18 models (see Figures~\ref{fig:fourier-resnet18-80-percent}, \ref{fig:fourier-resnet18-90-percent} and \ref{fig:fourier-resnet18-95-percent}). By comparing the heatmaps of rewinding and initialization based pruning methods to baselines, we find that these models are more resilient to perturbations of varying severity. 

\subsection{Additional Conv8 heatmaps}
In the Conv8 models, differences in the heatmaps of initialization methods and the baseline model persist up to the highest sparsity level of 95\%, as seen in Figure~\ref{fig:fourier-conv8-95-percent}.
The top three rows in each figure provide the Fourier heatmaps for each model while the bottom three rows provide the difference to the dense baseline. 
In the difference heatmaps, blue pixels are where the compressed model has an error rate lower than the dense model and red pixels are where the compressed model has an error rate higher than the dense model.

\begin{figure}[h]
\centering
  \begin{tabular}{@{}c@{}}
    \includegraphics[width=\textwidth]{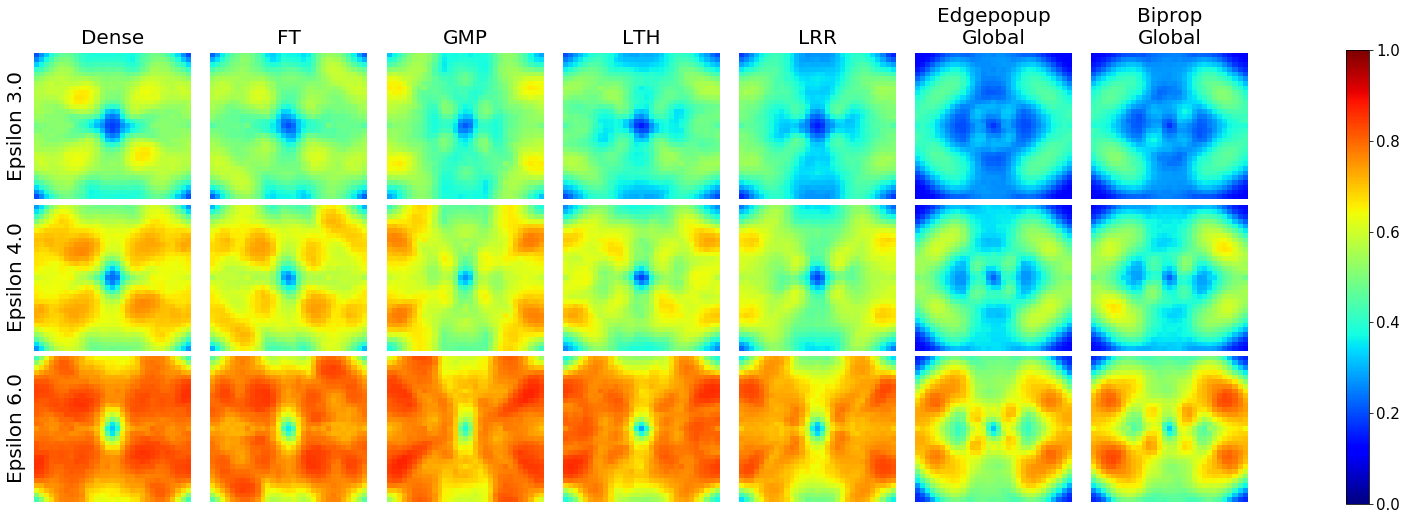} \\
    \includegraphics[width=\textwidth]{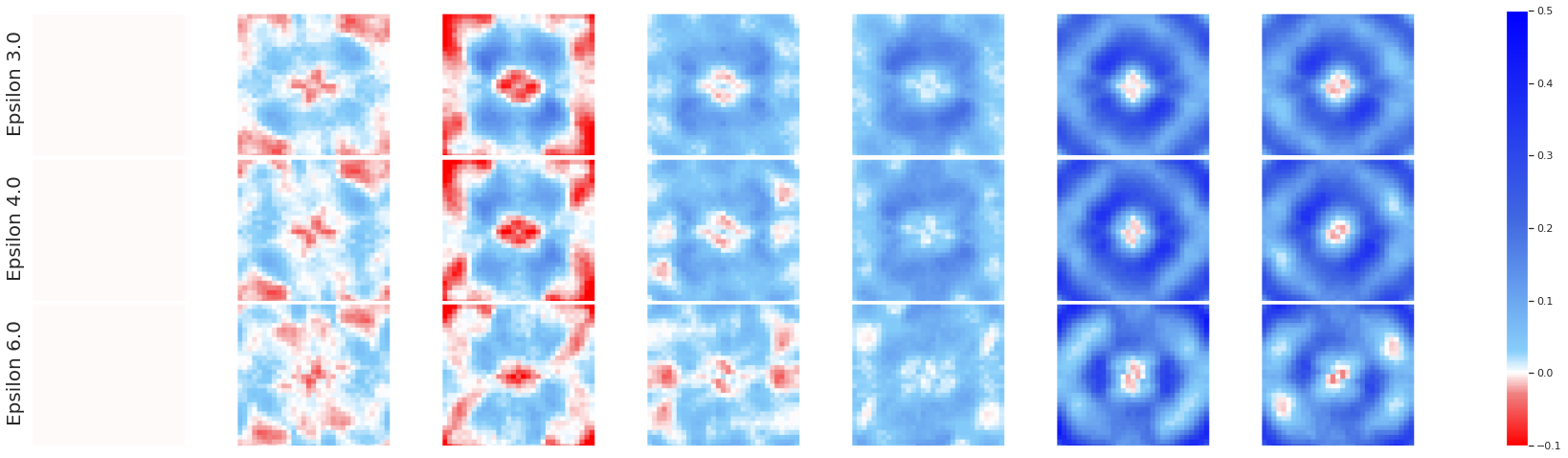}
  \end{tabular}
  \caption{{\bfseries Visualizing the response of compressed models to perturbations at different frequencies}: The top three rows are Fourier heatmaps for error rate of Conv8 models trained on CIFAR-10 with 90\% of weights pruned. The bottom three rows are difference to the baseline with blue regions indicating lower error rate than baseline.}
  \label{fig:fourier-conv8-90-percent}
\end{figure}

\begin{figure}[h]
\centering
  \begin{tabular}{@{}c@{}}
    \includegraphics[width=\textwidth]{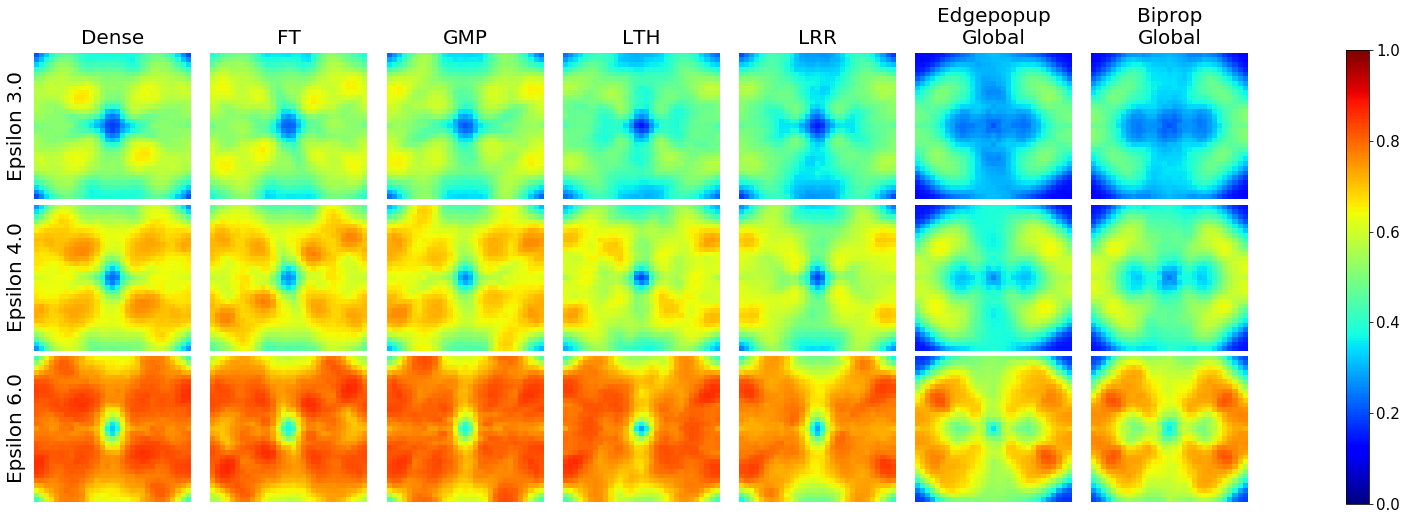} \\
    \includegraphics[width=\textwidth]{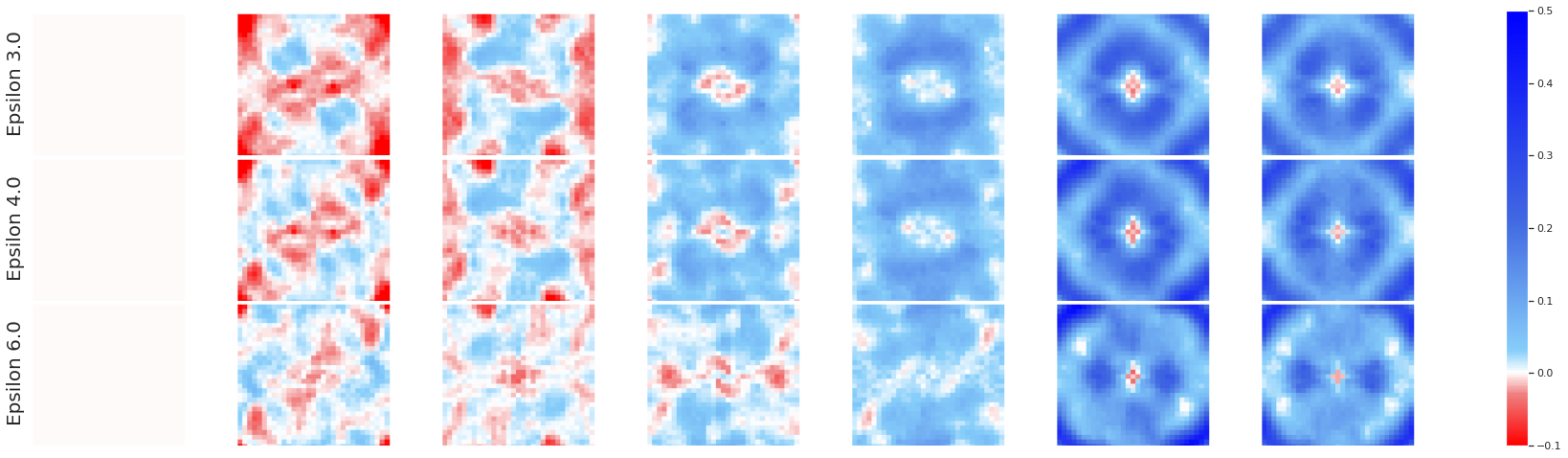}
  \end{tabular}
  \caption{{\bfseries Visualizing the response of compressed models to perturbations at different frequencies}: The top three rows are Fourier heatmaps for error rate of Conv8 models trained on CIFAR-10 with 95\% of weights pruned. The bottom three rows are difference to the baseline with blue regions indicating lower error rate than baseline.}
  \label{fig:fourier-conv8-95-percent}
\end{figure}

\subsection{ResNet-18 heatmaps}
Here we provide Fourier error rate heatmaps for the ResNet-18 architecture trained using different pruning methods. 
As in the Conv8 heatmap figures, we include heatmaps for a trained dense ResNet-18 model for reference and the difference heatmaps clearly conveying the difference of each compression method to the baseline.

\begin{figure}[h]
\centering
  \begin{tabular}{@{}c@{}}
    \includegraphics[width=\textwidth]{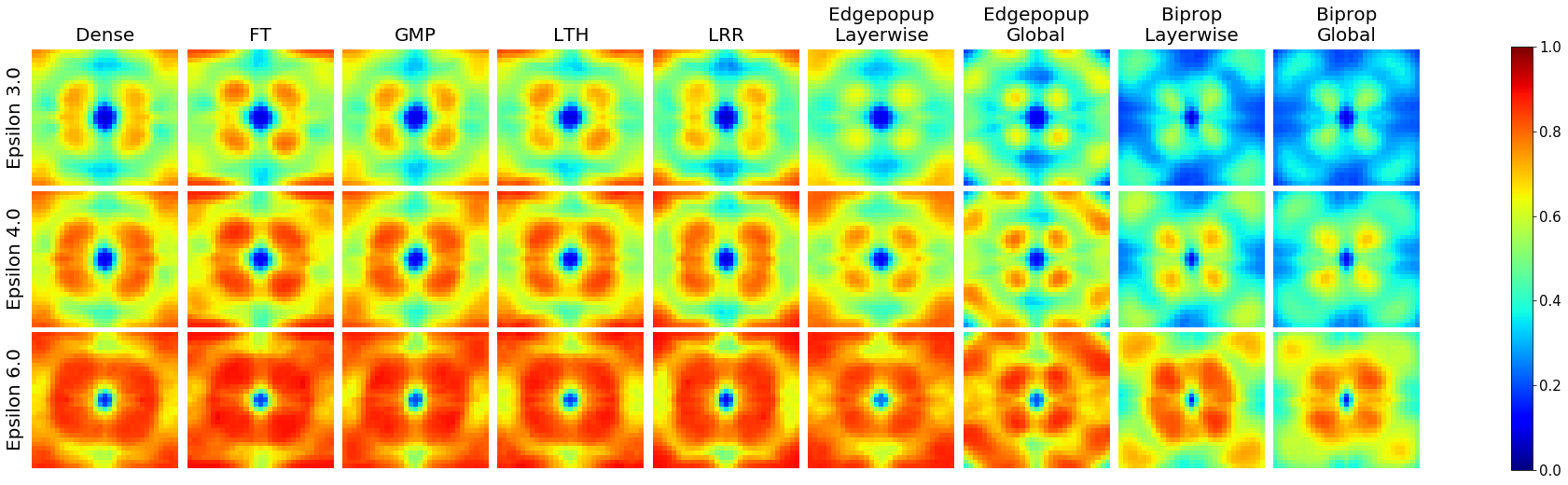} \\
    \includegraphics[width=\textwidth]{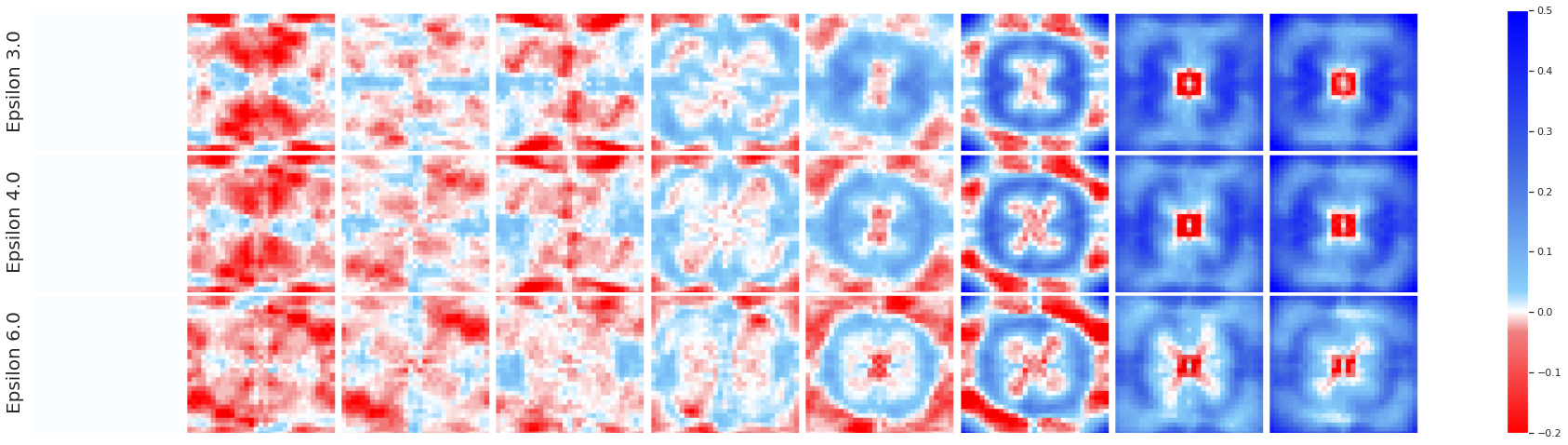}
  \end{tabular}
  \caption{{\bfseries Visualizing the response of compressed models to perturbations at different frequencies}: The top three rows are Fourier heatmaps for error rate of ResNet-18 models trained on CIFAR-10 with 80\% of weights pruned. The bottom three rows are difference to the baseline with blue regions indicating lower error rate than baseline.}
  \label{fig:fourier-resnet18-80-percent}
\end{figure}

\begin{figure}[h]
\centering
  \begin{tabular}{@{}c@{}}
    \includegraphics[width=\textwidth]{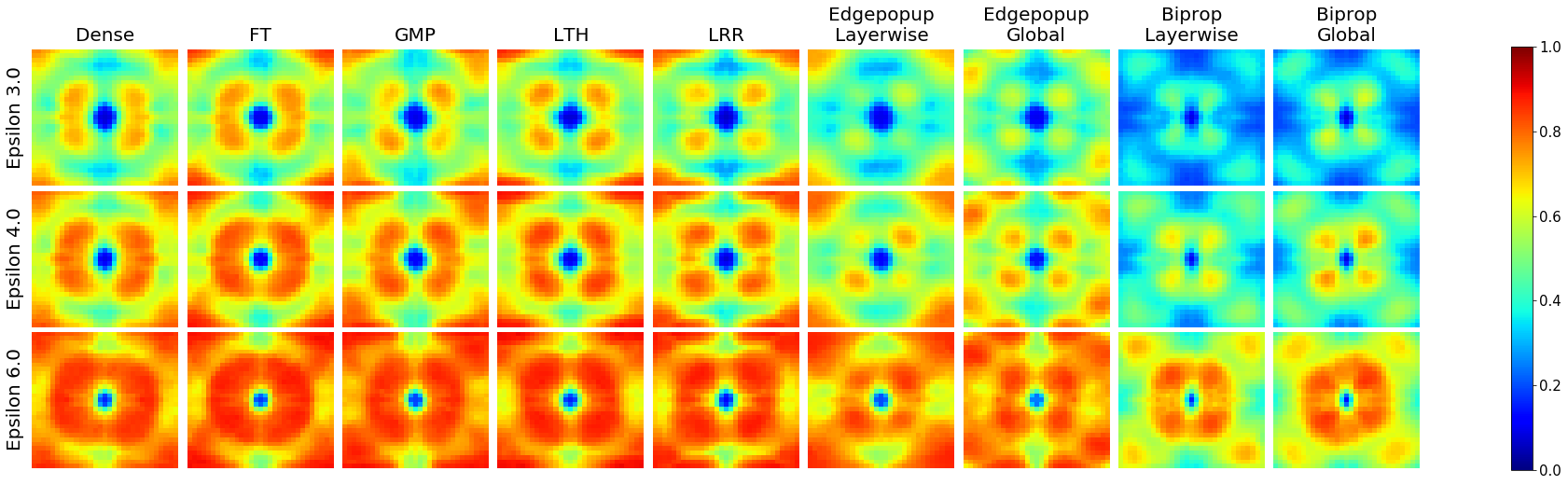} \\
    \includegraphics[width=\textwidth]{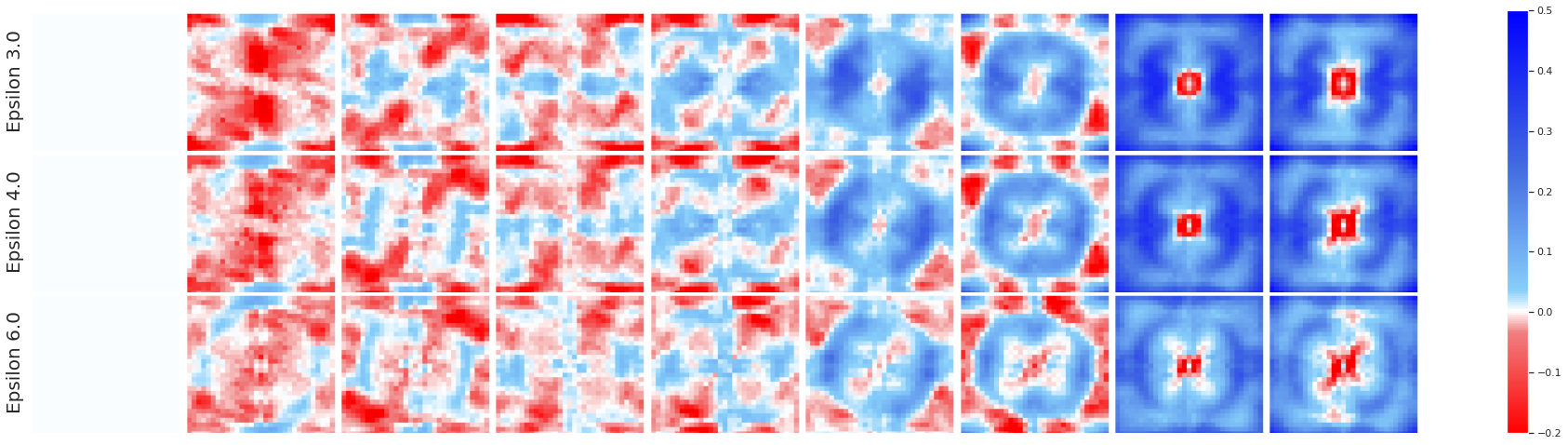}
  \end{tabular}
  \caption{{\bfseries Visualizing the response of compressed models to perturbations at different frequencies}: The top three rows are Fourier heatmaps for error rate of ResNet-18 models trained on CIFAR-10 with 90\% of weights pruned. The bottom three rows are difference to the baseline with blue regions indicating lower error rate than baseline.}
  \label{fig:fourier-resnet18-90-percent}
\end{figure}

\begin{figure}[h]
\centering
  \begin{tabular}{@{}c@{}}
    \includegraphics[width=\textwidth]{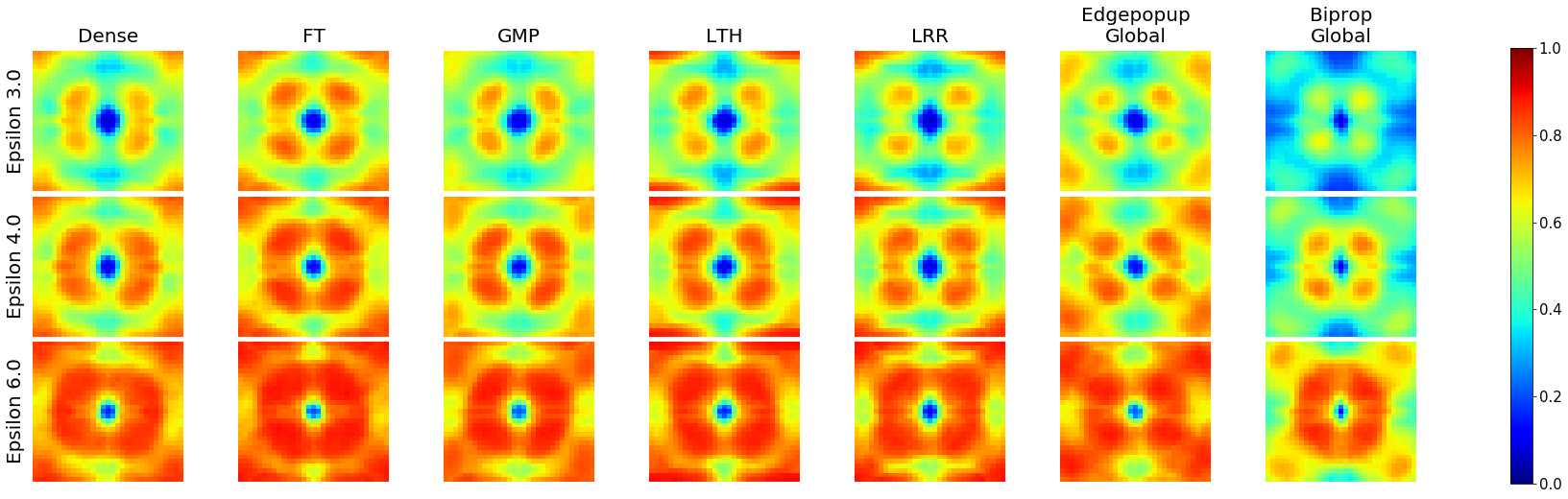} \\
    \includegraphics[width=\textwidth]{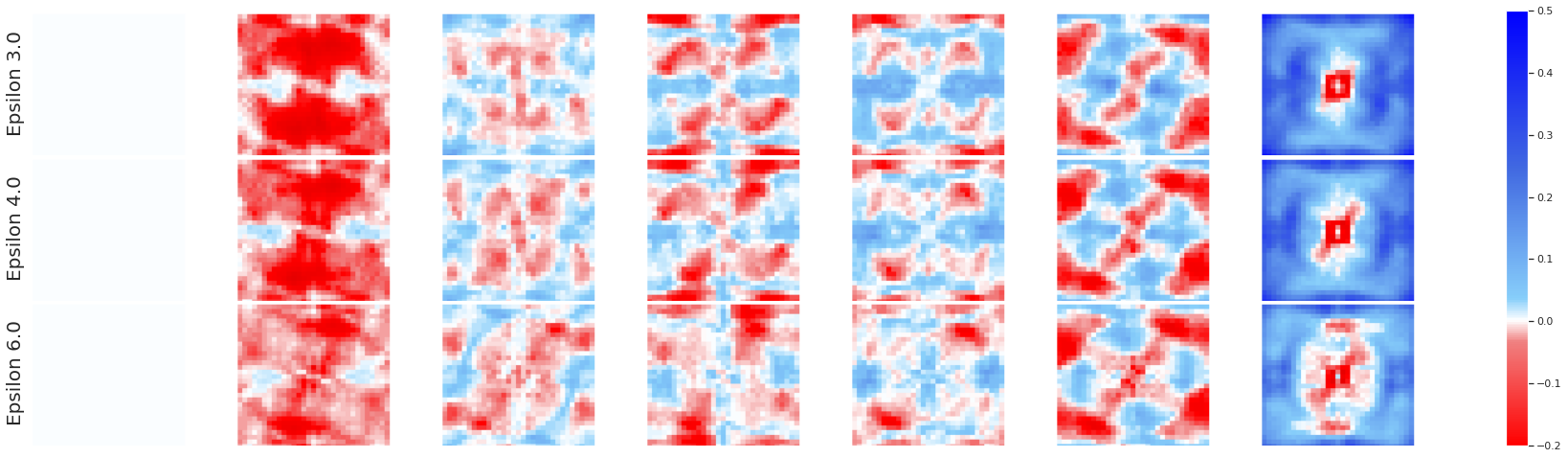}
  \end{tabular}
  \caption{{\bfseries Visualizing the response of compressed models to perturbations at different frequencies}: The top three rows are Fourier heatmaps for error rate of ResNet-18 models trained on CIFAR-10 with 95\% of weights pruned. The bottom three rows are difference to the baseline with blue regions indicating lower error rate than baseline.}
  \label{fig:fourier-resnet18-95-percent}
\end{figure}

\section{Constructing iterative pruning with rewinding from fine-tuning}
\label{appendix:n_shot}

While LTH \citep{frankle2018lottery} and LRR \citep{renda2020comparing} offer unsurpassed performance, such approaches also greatly extend the training duration, pruning just 20\% of the remaining weights every $T-r$ epochs, where $T$ is the initial training duration and $r$ is the epoch the weights/learning-rate are rewound to after each pruning event (here, $r=12$ and $T=160$). This raises the question: {\em Is longer training and the multi-shot pruning procedure critical to the robustness improvements LTH/LRR offer relative to FT/GMP?} 

To test this, we gradually construct the LTH/LRR pruning approaches used in this paper by starting from a fine-tuning approach and adding modifications until we produce the LTH/LRR method that prunes the network 13 times to reach 95\% sparsity. The phases of this construction for LRR are illustrated in Figure~\ref{fig:fourier-shots}, wherein we plot a column of Fourier heatmaps for each phase. Specifically, the first column is our FT approach, the second column extends the fine-tuning duration, the third column adds learning-rate rewinding to this fine-tuning period, the fourth column decreases the iterative prune rate to achieve 95\% sparsity in 4 shots rather than 1, and subsequent columns continue to increase the number of pruning shots. In this construction process, we find a notable benefit of adding learning-rate rewinding (70.9\% to 72.6\% CIFAR-10-C accuracy moving from column 2 to column 3), but the biggest benefits of LTH/LRR come from combining this rewinding with multiple iterations (i.e., all columns from 4 onward display at least 75\% robust accuracy). Interestingly, our results also indicate that it may be possible to achieve the robustness benefits of LTH/LRR with a higher iterative pruning rate and thus fewer pruning shots/iterations than what is standard in the literature \citep{frankle2018lottery,renda2020comparing}. 

\begin{figure}[t]
    \centering
    \includegraphics[width=\textwidth]{"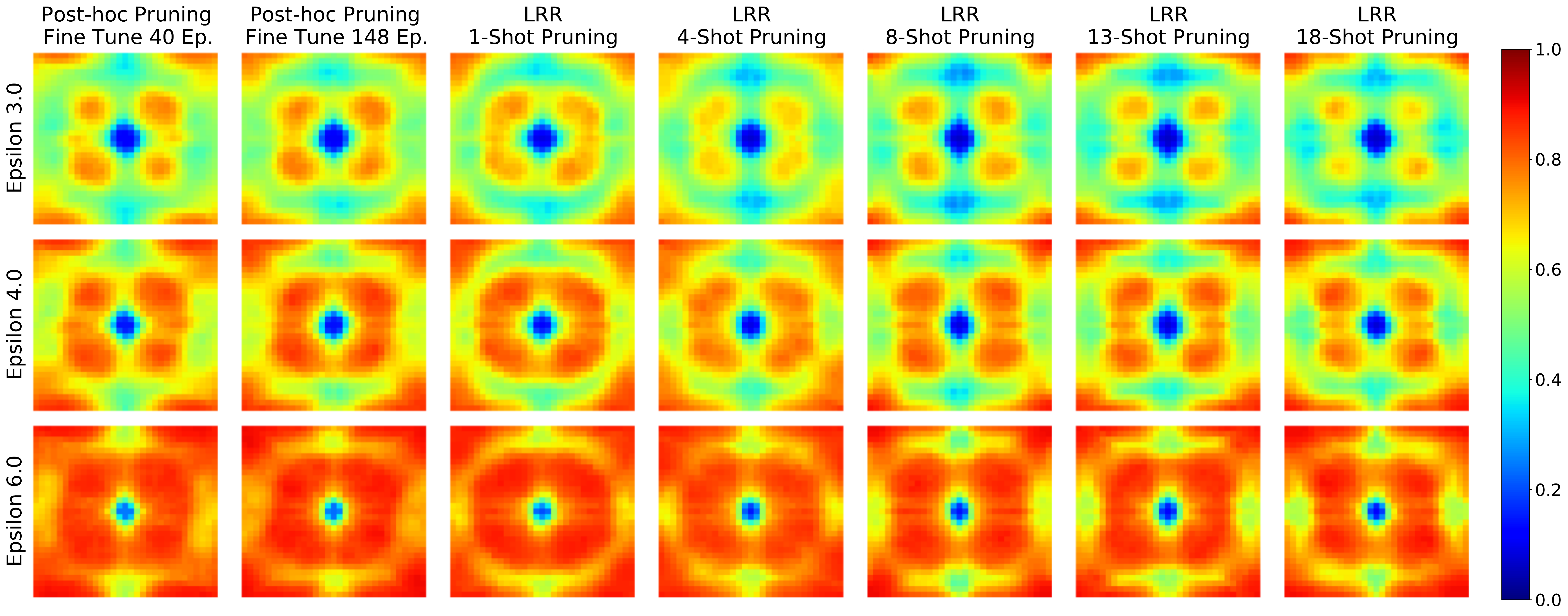"}
    \caption{{\bfseries Comparing the resiliencies of \RWD and Traditional methods to perturbations at different frequencies}: Fourier heatmaps for error rate of ResNet18 models trained on CIFAR-10 in which 95\% of the weights are pruned. 
    }
    \label{fig:fourier-shots}
\end{figure}
We now repeat this experiment using 90\% sparsity (Figure~\ref{fig:fourier-shots-90}), using LTH instead of LRR at 95\% sparsity (Figure~\ref{fig:fourier-shots-lth}), and using LTH and 90\% sparsity (Figure~\ref{fig:fourier-shots-lth-90}). 

At 95\% sparsity, we observe the same pattern: adding multiple shots of pruning is critical to improving the LTH heatmaps and robustnesses of the rewinding-based methods (Figure~\ref{fig:fourier-shots-lth}). That is to say, adding rewinding and a longer post-pruning fine-tuning duration to our FT method is not sufficient to obtain the results achievable with LTH/LRR---multiple iterations are needed. Interestingly, as especially visible at epsilon 6.0 in the Fourier heatmaps, LRR (Figure~\ref{fig:fourier-shots}) is clearly more resilient to perturbations than LTH, which is consistent with the improved performance of LRR relative to LTH.

\begin{figure}
    \centering
    \includegraphics[width=\textwidth]{"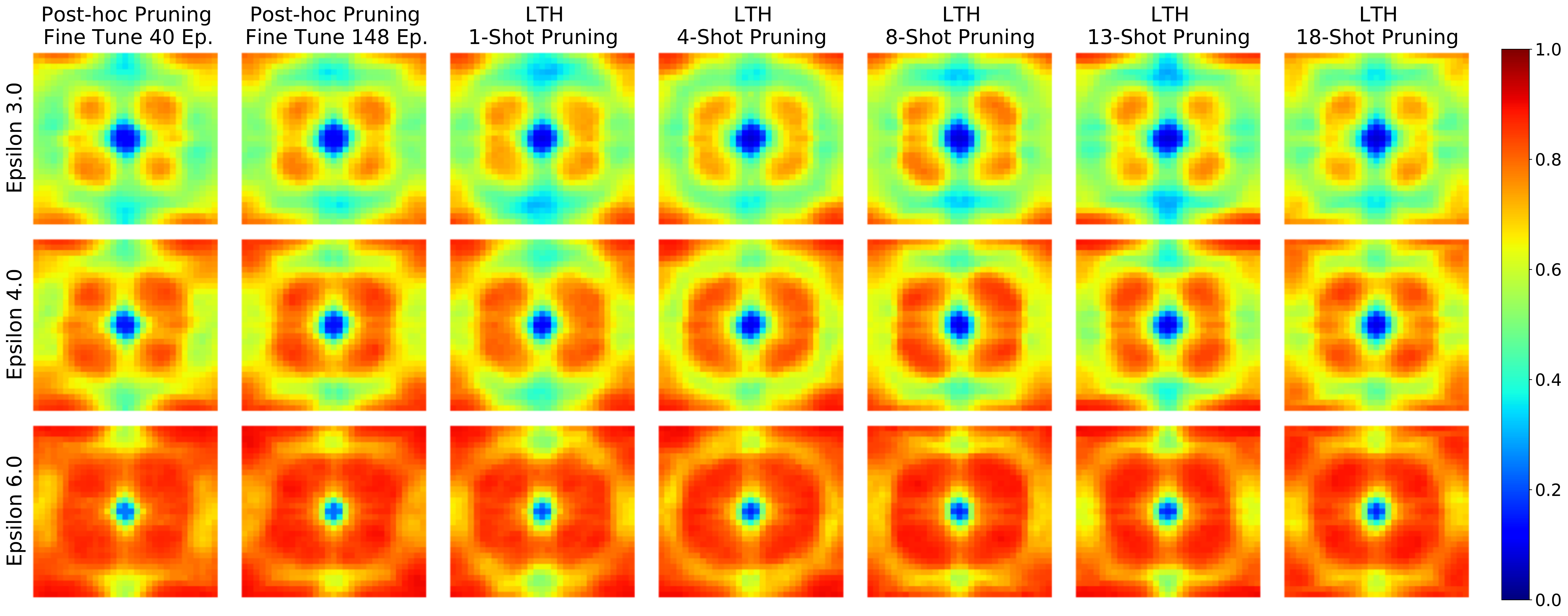"}
    \caption{{\bfseries Comparing the resiliencies of \RWD and Traditional methods to perturbations at different frequencies}: Fourier heatmaps for error rate of ResNet18 models trained on CIFAR-10 in which 95\% of the weights are pruned via LTH. 
    }
    \label{fig:fourier-shots-lth}
\end{figure}

At 90\% sparsity, for both LTH (Figure~\ref{fig:fourier-shots-lth-90}) and LRR (Figure~\ref{fig:fourier-shots-90}), the Fourier heatmaps reflect benefits of multiple shots and rewinding (particularly near the centers of the images for all epsilons). For LRR, there is greater similarity among the \RWD and \INIT Fourier heatmaps at 90\% sparsity than at 95\% sparsity, and this is reflected in their robustnesses in the captions, which are less separated in the 90\% sparsity case. Notably, however, all these robustness figures are consistent with the aforementioned heatmap improvements in that they show the benefits of combining rewinding with multiple pruning shots. Note that 10-shot pruning corresponds to the scheme / iterative pruning rate (20\%) we use to reach 90\% sparsity in other sections (e.g., Figure~\ref{fig:fourier-resnet18-90-percent}).

\begin{figure}
    \centering
    \includegraphics[width=\textwidth]{"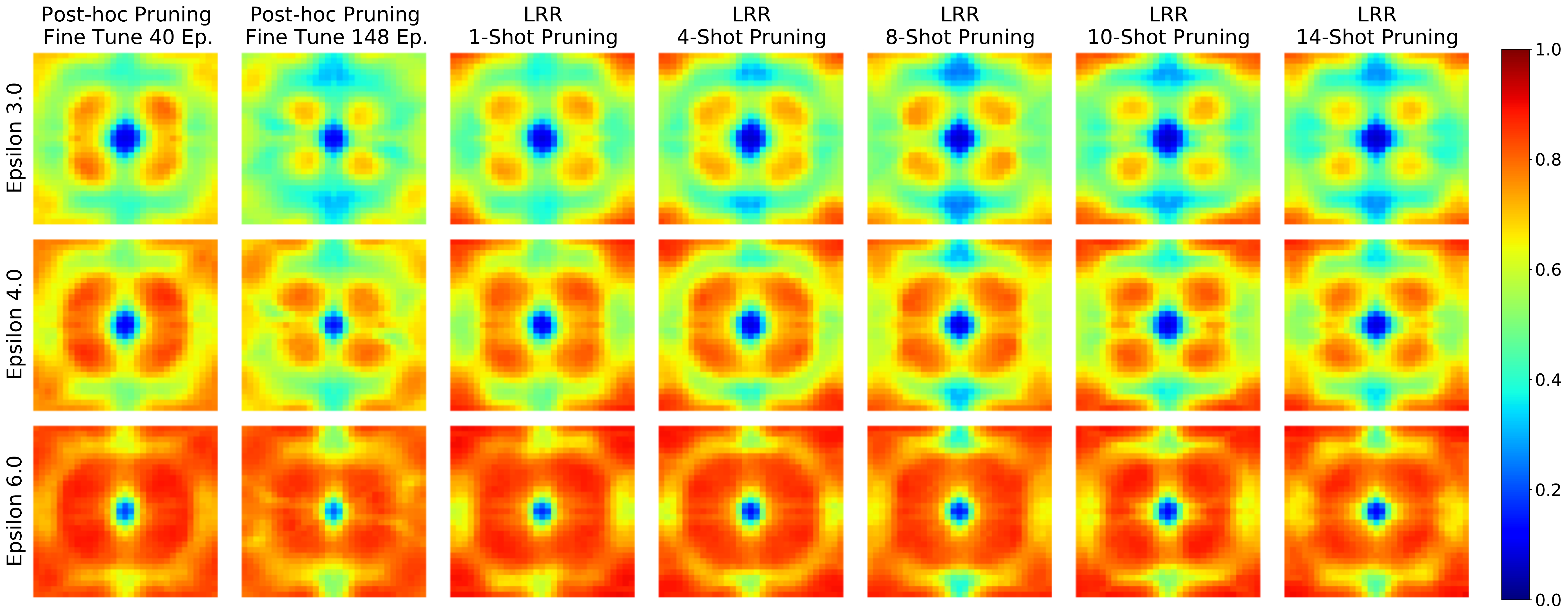"}
    \caption{{\bfseries Comparing the resiliencies of \RWD and Traditional methods to perturbations at different frequencies}: Fourier heatmaps for error rate of ResNet18 models trained on CIFAR-10 in which 90\% of the weights are pruned via LRR. 
    }
    \label{fig:fourier-shots-90}
\end{figure}
\begin{figure}
    \centering
    \includegraphics[width=\textwidth]{"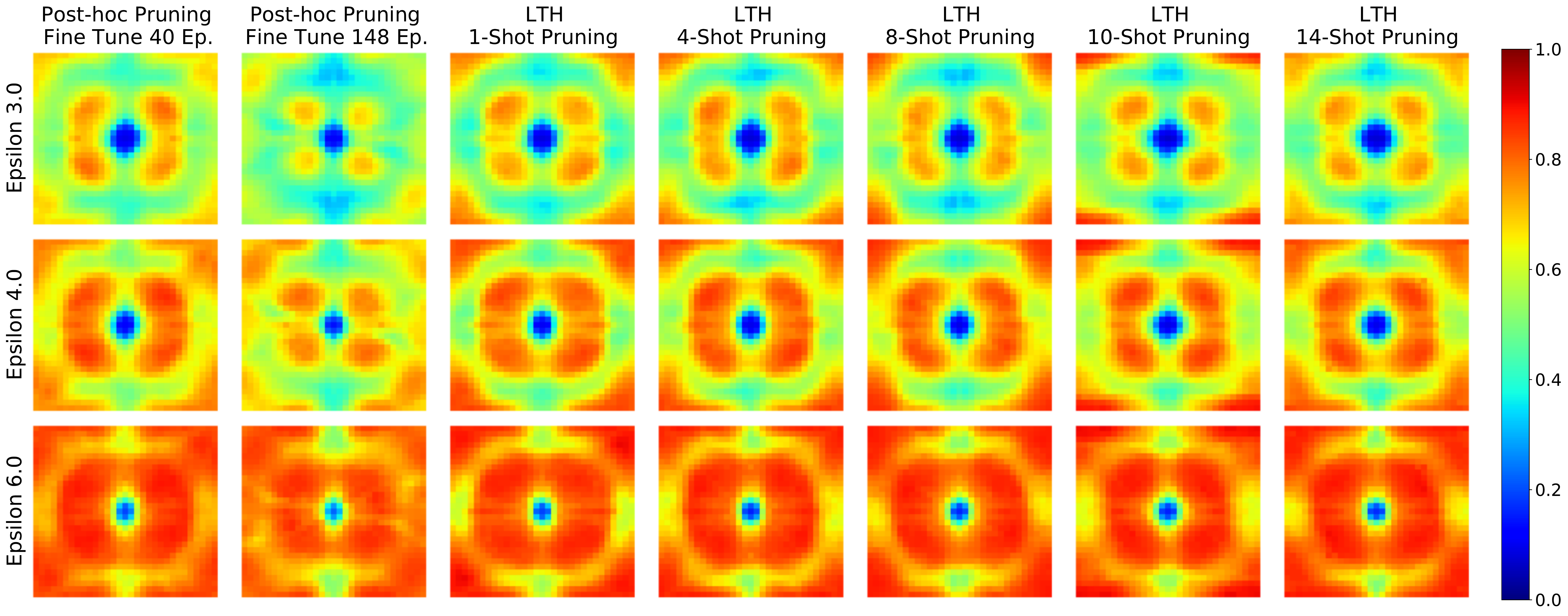"}
    \caption{{\bfseries Comparing the resiliencies of \RWD and Traditional methods to perturbations at different frequencies}: Fourier heatmaps for error rate of ResNet18 models trained on CIFAR-10 in which 90\% of the weights are pruned via LTH. 
    }
    \label{fig:fourier-shots-lth-90}
\end{figure}

\section{Additional results with CARDS and CARD-Deck}
\label{appendix:card-deck}
In this section, we provide tables for all experimental results from Section~\ref{sec:sota}.
This includes tables for individual \cnets on CIFAR-10 for ResNet-18 (Table~\ref{table:aug-card}), ResNeXt-29 (Table~\ref{table:resnext29-card}), ResNet-50 (Table~\ref{table:resnet50-card}), and WideResNet-18-2 (Table~\ref{table:wideresnet18-card}). 
Additionally, we provide tables for CIFAR-10 \cdecks using ResNet-18 (Table~\ref{table:card-deck-resnet18}), WideResNet-18-2 (Table~\ref{table:card-deck-wideresnet}),
and CIFAR-100 \cdecks using WideResNet-18-2 (Table~\ref{table:c100-card-deck-wideresnet}).
Breakdowns for the performance of CIFAR-10  ResNet-18 \cnets and \cdecks on each of the 15 corruption types in CIFAR-10-C are provided in Tables~\ref{table:cards-80} -- Tables~\ref{table:card-decks-95}. 
As a reference, tables for individual \cnets provide results for dense baseline models. 
Due to the structure of the table these results are intentionally repeated at each sparsity level (the dense baselines are not pruned so their performance remains constant).

\subsection{Tables of CARD and CARD-Decks results for ResNet-18 and WideResNet-18}
The clean and robust accuracies (averaged across three realizations) of \cnets for each pruning scheme are provided in Table~\ref{table:aug-card}.
We find that \cnets perform comparably to (and in some cases better than) their dense counterparts in terms of accuracy and robustness but have a significantly smaller memory footprint.

\begin{table}
\tiny
\centering
\setlength\tabcolsep{3 pt}
\begin{tabular}{ccccccccccccccccccccccc}
\toprule
{} & & \multicolumn{3}{c}{Baseline} & \multicolumn{18}{c}{\cnet} \\ \cmidrule(lr{4pt}){3-5}\cmidrule(lr){6-23}
{} & & \multicolumn{3}{c}{Dense} & \multicolumn{6}{c}{Edgepopup} & \multicolumn{3}{c}{LRR} & \multicolumn{3}{c}{LTH} & \multicolumn{6}{c}{Biprop} \\ \cmidrule(lr){3-5}\cmidrule(lr){6-11}\cmidrule(lr){12-14}\cmidrule(lr){15-17}\cmidrule(lr){18-23}
{} & & \multicolumn{3}{c}{-} & \multicolumn{3}{c}{Layerwise} & \multicolumn{3}{c}{Global} & \multicolumn{3}{c}{Global} & \multicolumn{3}{c}{Global} & \multicolumn{3}{c}{Layerwise} & \multicolumn{3}{c}{Global}\\ \cmidrule(lr){3-5}\cmidrule(lr){6-8}\cmidrule(lr){9-11}\cmidrule(lr){12-14}\cmidrule(lr){15-17}\cmidrule(lr){18-20}\cmidrule(lr){21-23}
{} &  &      \rotatebox{90}{Augmix} &         \rotatebox{90}{Clean} &      \rotatebox{90}{Gauss.} &           \rotatebox{90}{Augmix} &         \rotatebox{90}{Clean} &      \rotatebox{90}{Gauss.} &           \rotatebox{90}{Augmix} &         \rotatebox{90}{Clean} &      \rotatebox{90}{Gauss.} &         \rotatebox{90}{Augmix} &         \rotatebox{90}{Clean} &      \rotatebox{90}{Gauss.} &         \rotatebox{90}{Augmix} &         \rotatebox{90}{Clean} &      \rotatebox{90}{Gauss.} &           \rotatebox{90}{Augmix} &         \rotatebox{90}{Clean} &      \rotatebox{90}{Gaussian} &          \rotatebox{90}{Augmix} &         \rotatebox{90}{Clean} &      \rotatebox{90}{Gauss.} \\
\midrule
\multirow{3}{*}{\rotatebox[origin=c]{90}{\makecell{80\%}}} & Clean Acc.  &   \textbf{95.5} &  95.1 &     93.9 &     94.3 &  93.7 &     92.4 &      \textbf{94.9}  &  94.4 &       93 &   96.1 &  95.6 &     93.8 &   95.6 &  94.9 &     93.5 &   93.9 &  93.7 &     92.4 &      \textbf{94.5} &  94.1 &     93.2 \\
&Robust Acc. &     \textbf{89.2} &  73.7 &     85.6 &     87.8 &  76.7 &     85.1 &      \textbf{88.4}  &  74.4 &     85.7 &   89.8 &  75.7 &     86.4 &   89.4 &  74.4 &     85.8 &   87.5 &  76.1 &     85.1 &      \textbf{87.8} &  74.3 &     85.3 \\
&Memory (Mbit)      &      \textbf{358}  &   358 &      358 &     2.23 &  2.23 &     2.23 &      2.23 &  2.23 &     2.23 &   71.5 &  71.5 &     71.5 &   71.5 &  71.5 &     71.5 &   2.23 &  2.23 &     2.23 &      2.23 &  2.23 &     2.23 \\
\midrule
\multirow{3}{*}{\rotatebox[origin=c]{90}{\makecell{90\%}}} & Clean Acc.  &     \textbf{95.5} &  95.1 &     93.9 &          94.4 &         93.9 &            92.9 &             94.4 &            94.1 &               92.8 &   \textbf{96.3} &  95.6 &     93.9 &   95.7 &  95.2 &     93.6 &          \textbf{94.4} &         94.1 &            92.7 &             93.7 &            93.6 &                 93 \\
& Robust Acc. &     \textbf{89.2} &  73.7 &     85.6 &            \textbf{88}  &         75.6 &            85.2 &             87.9 &              76 &               85.4 &   \textbf{89.8} &  76.1 &     86.3 &   89.7 &  74.3 &       86 &          \textbf{87.8}  &         75.1 &              85 &             87.1 &            74.6 &               84.2 \\
& Memory (Mbit)      &      \textbf{358} &   358 &      358 &          1.12 &         1.12 &            1.12 &             1.12 &            1.12 &               1.12 &   35.8 &  35.8 &     35.8 &   35.8 &  35.8 &     35.8 &          1.12 &         1.12 &            1.12 &             1.12 &            1.12 &               1.12 \\
\midrule
\multirow{3}{*}{\rotatebox[origin=c]{90}{\makecell{95\%}}} & Clean Acc.  &     \textbf{95.5} &  95.1 &     93.9 &          \textbf{94.5}  &           94 &            92.6 &             93.2 &            92.7 &               91.2 &   96.1 &  95.7 &     93.9 &   \textbf{95.8} &  95.1 &     93.8 &          94.2 &         93.8 &            92.5 &             92.5 &            92.1 &               91.4 \\
& Robust Acc. &     \textbf{89.2} &  73.7 &     85.6 &          87.8 &         73.1 &            84.3 &             85.7 &            73.4 &               83.9 &   89.6 &  75.6 &     86.3 &   \textbf{89.7} &  74.3 &       86 &          87.5 &         73.8 &            84.4 &             84.7 &            73.4 &               83.1 \\
& Memory (Mbit)      &      \textbf{358} &   358 &      358 &          \textbf{0.56} &         0.56 &            0.56 &            \textbf{0.56} &            0.56 &               0.56 &   \textbf{17.9} &  17.9 &     17.9 &   \textbf{17.9} &  17.9 &     17.9 &         \textbf{0.56} &         0.56 &            0.56 &             \textbf{0.56}  &            0.56 &               0.56 \\
\bottomrule
\end{tabular}
\vspace{2mm}
\caption{Performance comparison between dense baselines and \cnets using ResNet-18 architecture. Clean and Robust Acc. refer to accuracy on CIFAR-10 and CIFAR-10-C, respectively. The best performance for each method is shown in bold. }
\label{table:aug-card}
\end{table}

\begin{table}
\tiny
\centering
\setlength\tabcolsep{3 pt}
\begin{tabular}{ccccccccccccccccccccccc}
\toprule
{} & {} & \multicolumn{12}{c}{\cdeck (Agnostic/Adaptive)} \\\cmidrule(lr){3-14}
{} & {}  & \multicolumn{3}{c}{Edgepopup (Global)} & \multicolumn{3}{c}{LRR} & \multicolumn{3}{c}{LTH} & \multicolumn{3}{c}{Biprop (Global)} \\\cmidrule(lr){3-5}\cmidrule(lr){6-8}\cmidrule(lr){9-11}\cmidrule(lr){12-14}
{} & {} & 2 & 4 & 6 & 2 & 4 & 6 & 2 & 4 & 6 & 2 & 4 & 6 \\
\midrule
\multirow{3}{*}{\rotatebox[origin=c]{90}{\makecell{80\%}}} & Clean Acc.  &           92.1/94.1 &          94/94.5 &             94.2/94.8 &       96/96 &    96.3/96.4 &    96.4/96.6 &    95.5/95.5 &      96/96.1 &    96.1/96.1 &     93.8/93.8 &     94.3/94.3 &     94.3/94.4 \\
& Robust Acc. &           85/88.9 &          88.4/89.8 &             88.9/90 &     89.7/90.9 &    91.7/91.8 &    91.9/92 &    89.4/90.4 &      91/91.2 &    91.2/91.4 &     87.5/88.6 &     88.8/89.3 &       89/89.6 \\
& Memory (Mbit)     &         4.47 &  8.94  &  13.4 &   143 &  286 &  429 &  143 &  286 &  429 &  4.47 &  8.94  &  13.4  \\
\midrule
\multirow{3}{*}{\rotatebox[origin=c]{90}{\makecell{90\%}}} & Clean Acc.  &           92.9/94.4 &          94.6/94.8 &             94.7/94.8 &     96.3/96.3 &    96.4/96.4 &    96.4/\textbf{96.6} &     95.7/95.7 &    95.9/95.7 &    96.2/\textbf{96.2} &       94/94 &     94.6/94.4 &     94.5/94.6 \\
& Robust Acc. &           85.2/89.2 &          88.6/90.1 &             89.3/\textbf{90.4} &     89.8/91.1 &    91.7/91.8 &      92/\textbf{92.1} &     89.4/90.5 &      91/91.3 &    91.3/\textbf{91.5} &     87.4/88.6 &     89.2/89.5 &     89.4/\textbf{89.9} \\
& Memory (Mbit)     &         2.23 &  4.47 &  6.70 &  71.5 &  143 &  215 &  71.5 &  143 &  215 &  2.23 &  4.47 &  6.70 \\
\midrule 
\multirow{3}{*}{\rotatebox[origin=c]{90}{\makecell{95\%}}} & Clean Acc.  &           92.6/94.5 &          94.2/94.8 &             94.5/\textbf{95.1} &     96.1/96.1 &     96.3/96.4 &    96.3/96.5 &     95.8/95.8 &       96/96.1 &    96.1/96.2 &       94/94 &     94.7/94.5 &     \textbf{94.7}/94.6 \\
& Robust Acc. &           84.3/88.6 &          87.7/89.5 &             88.4/89.9 &     89.6/91 &     91.6/91.8 &    91.9/92 &       89/90.3 &     90.8/91.1 &    91.1/91.4 &     87.2/88.5 &     88.9/89.4 &     89.2/89.7 \\
& Memory (Mbit)     &         \textbf{1.12} &  2.23 &  3.35 &  \textbf{35.8} &  71.5 &  107 &  \textbf{35.8} &  71.5 &  107 &  \textbf{1.12} &  2.23 &  3.35 \\
\bottomrule
\end{tabular}
\vspace{2mm}
\caption{Performance of domain-agnostic and domain-adaptive ResNet-18 \cdecks. Clean and Robust Acc. refer CIFAR-10 and CIFAR-10-C accuracy, respectively. }
\label{table:card-deck-resnet18}
\end{table}

\begin{table}
\tiny
\centering
\setlength\tabcolsep{3 pt}
\begin{tabular}{ccccccccccccccccccccccc}
\toprule
{} & {} & \multicolumn{12}{c}{\cdeck (Agnostic/Adaptive)} \\
\cmidrule(lr){3-14}
{} & {}  & \multicolumn{3}{c}{Edgepopup (Global)} & \multicolumn{3}{c}{LRR} & \multicolumn{3}{c}{LTH} & \multicolumn{3}{c}{Biprop (Global)} \\\cmidrule(lr){3-5}\cmidrule(lr){6-8}\cmidrule(lr){9-11}\cmidrule(lr){12-14}
{} & {} & 2 & 4 & 6 & 2 & 4 & 6 & 2 & 4 & 6 & 2 & 4 & 6 \\
\midrule
\multirow{3}{*}{\rotatebox[origin=c]{90}{\makecell{90\%}}} & Clean Acc.  &           
92.4/         92.9 &
94.0/         94.8 &
94.8/         95.1 &
96.3/         96.3 &
96.7/         96.7 &
96.7/         96.8 &
96.1/         96.1 &
96.5/         96.4 &
96.6/         96.6 &
92.4/         93.1 &
94.3/         94.5 &
94.9/         95.0 \\
& Robust Acc. &      
85.1/         86.2 &
88.6/         90.0 &
90.1/         90.6 &
90.6/         91.7 &
92.3/         92.3 &
92.5/         92.6 &
90.1/         91.2 &
91.6/         91.8 &
91.9/         92.1 &
85.3/         86.2 &
88.7/         89.8 &
89.9/         90.5 \\
& Memory (Mbit) &
8.93 & 17.86 & 26.79 & 
285.8 & 571.6 & 857.5 &
285.8 & 571.6 & 857.5 &
8.93 & 17.86 & 26.79
\\
\midrule 
\multirow{3}{*}{\rotatebox[origin=c]{90}{\makecell{95\%}}} & Clean Acc.  &           
         92.8/         93.4 &
         94.6/         94.9 &
         95.1/         \textbf{95.3} &
         96.3/         96.3 &
         96.8/         96.8 &
         96.6/\textbf{96.8} &
         96.1/         96.1 &
         96.5/         96.4 &
         96.6/         \textbf{96.7} &
         92.3/         92.8 &
         94.2/         94.5 &
         95.0/         \textbf{95.2} \\
& Robust Acc. &           
85.2/         86.1 &
88.6/         89.9 &
90.0/         \textbf{90.6} &
90.8/         91.8 &
92.4/         92.5 &
92.7/\textbf{92.75} &
89.9/         91.4 &
91.6/         91.9 &
91.9/         \textbf{92.2} &
84.9/         86.0 &
88.4/         89.5 &
90.0/         \textbf{90.5} \\
& Memory (Mbit) &
\textbf{4.46} & 8.93 & 13.39 &
\textbf{142.9} & 285.8 & 428.7 &
\textbf{142.9} & 285.8 & 428.7 &
\textbf{4.46} & 8.93 & 13.39 \\
\bottomrule
\end{tabular}
\vspace{2mm}
\caption{Performance of domain-agnostic and domain-adaptive WideResNet-18 \cdecks. Clean and Robust Acc. refer to CIFAR-10 and CIFAR-10-C accuracy, respectively.  The best performance for each method is shown in bold. For reference, dense WideResNet-18 (AugMix) model achieves (Clean Acc., Robust Acc., Memory) = (95.6\%, 89.3\%, 1429 Mbit).}
\label{table:card-deck-wideresnet}
\end{table}

\begin{table}
\tiny
\centering
\setlength\tabcolsep{3 pt}
\begin{tabular}{ccccccccccccccccccccccc}
\toprule
{} & {} & \multicolumn{12}{c}{\cdeck (Agnostic/Adaptive)} \\
\cmidrule(lr){3-14}
{} & {}  & \multicolumn{3}{c}{Edgepopup (Global)} & \multicolumn{3}{c}{LRR} & \multicolumn{3}{c}{LTH} & \multicolumn{3}{c}{Biprop (Global)} \\\cmidrule(lr){3-5}\cmidrule(lr){6-8}\cmidrule(lr){9-11}\cmidrule(lr){12-14}
{} & {} & 2 & 4 & 6 & 2 & 4 & 6 & 2 & 4 & 6 & 2 & 4 & 6 \\
\midrule
\multirow{3}{*}{\rotatebox[origin=c]{90}{\makecell{90\%}}} & Clean Acc.  &           
77.1/77.1 &          
78.5/78.4 &          
78.6/78.7 &
78.3/78.3 &
79.6/79.7 &
79.6/80.2 &
78.2/78.2 &
79.6/79.7 &
79.9/80.3 &
76.9/76.9 &
77.8/78.1 &          
78.0/\textbf{78.3} \\
& Robust Acc. &      
66.3/67.8 &
69.5/69.6 &
69.9/\textbf{70.3} &
66.9/68.8 &
70.5/70.7 &
71.0/71.2 &
66.2/68.6 &
69.8/70.4 &
70.6/71.0 &
65.8/67.3 &
68.7/69.0 &
69.1/\textbf{69.6} \\
& Memory (Mbit) &
8.95 & 17.90 & 26.85 & 
286.5 & 572.9 & 859.3 &
286.5 & 572.9 & 859.3 &
8.95 & 17.90 & 26.85
\\
\midrule 
\multirow{3}{*}{\rotatebox[origin=c]{90}{\makecell{95\%}}} & Clean Acc.  &           
77.1/77.1 &          
78.5/78.5 &          
78.7/\textbf{79.1} &
78.7/78.7 &          
79.9/80.2 &          
80.1/\textbf{80.6} & 
78.2/78.2 &          
79.7/80.0 &          
79.8/\textbf{80.4} & 
76.0/76.0 &
77.4/77.1 &          
77.9/77.8 \\
& Robust Acc. &           
65.6/67.1 &
68.9/69.0 &
69.4/69.7 &
67.1/68.8 &
70.6/70.7 &
71.1/\textbf{71.3} &
66.5/68.6 &
70.1/70.3 &
70.7/\textbf{71.0} &
64.8/66.5 &
67.9/68.1 &
68.4/68.7 \\
& Memory (Mbit) &
\textbf{4.48} & 8.95 & 13.43 &
\textbf{143.2} & 286.5 & 429.7 &
\textbf{143.2} & 286.5 & 429.7 &
\textbf{4.48} & 8.95 & 13.43 \\
\bottomrule
\end{tabular}
\vspace{2mm}
\caption{Performance of domain-agnostic and domain-adaptive WideResNet-18 \cdecks. Clean and Robust Acc. refer to CIFAR-100 and CIFAR-100-C accuracy, respectively.  The best performance for each method is shown in bold. For reference, dense WideResNet-18 (AugMix) model achieves (Clean Acc., Robust Acc., Memory) = (77.5\%, 66.1\%, 1433 Mbit).}
\label{table:c100-card-deck-wideresnet}
\end{table}

\subsection{Additional results for CARDs and CARD-Decks with ResNet-18}
\label{appendix:card-tables}
In this section, we provide a breakdown of the accuracy of ResNet-18 \cnets and \cdecks by CIFAR-10-C corruption types. In particular, Tables~\ref{table:cards-80} - \ref{table:cards-95} contain the performance of \cnets trained on clean, Augmix, and Gaussian augmentations when tested on CIFAR-10-C corruption types. Tables~\ref{table:card-decks-80} to \ref{table:card-decks-95} contain the performance of LTH, LRR, EP, and BP \cdecks on individual CIFAR-10-C corruptions.

As a note of interest, we found that the best performance on different CIFAR-10-C corruptions changes for individual \cnets as the sparsity level increases. At 80\% sparsity, a Gaussian \cnet yields the highest accuracy on impulse noise but at 90\% and 95\% sparsity levels Augmix \cnets deliver the highest accuracy on impulse noise. Further, at 95\% sparsity the margin of difference in accuracy on impulse noise provided by the Augmix \cnet over the Gaussian \cnet is more significant. 

While pruning using FT and GMP were unable to yield \cnets, note that we include the accuracy and robustness of ResNet-18 models pruned using FT and GMP with the same augmentation schemes in Table~\ref{table:resnet18-ft-gmp} for comparison against the performance of ResNet-18 \cnets in Table~\ref{table:aug-card}.


\subsection{Achieving state-of-the-art performance on CIFAR-10-C using larger models}
\label{appendix:larger_models}


We report these results in Tables \ref{table:resnext29-card}, \ref{table:resnet50-card} and \ref{table:wideresnet18-card}.
To summarize, our results highlight the fact that the accuracy/robustness gains due to the model compression (and ensembling) are compatible with the gains from the existing strategies, i.e., data augmentation and the use of larger models. By combining these strategies with the scheme proposed in this paper, we achieve even larger gains in terms of robustness and accuracy, in turn, establishing a new SOTA. Note that we include performance of WideResNet-18 \cdecks composed of layerwise pruned BP and EP models in Table~\ref{table:layerwise-agnostic-card-deck}.

\subsection{Note on gating function performance}

In Table~\ref{table:gating-performance}, we provide a break down of the performance of the spectral-similarity based gating function by CIFAR-10-C corruption type. For each augmentation scheme and corruption type, the corresponding number indicates the percetage of data from that corruption selected by the gating function averaged across the 5 severity levels in CIFAR-10-C. Based on the performance of Augmix and Gaussian \cnets by CIFAR-10-C corruption type in Tables~\ref{table:cards-80} -- \ref{table:cards-95}, entries in the table are marked in bold whenever a model pruned to sparsity 80\%, 90\%, or 95\% trained using that data augmentation scheme achieved the highest accuracy averaged over all severity levels of that corruption type. Bolding these entries in Table~\ref{table:gating-performance} indicates that the gating function typically selects the best performing augmentation scheme, and thereby the \cnets in the deck trained on data most similar to the incoming test data, for the domian-adaptive \cdecks. Improvements could be made by determining a gating function that is more accurate on the frost and jpeg corruptions. As noted in Section~\ref{appendix:card-tables}, the augmentation scheme yielding the best performing models on impulse and glass corruptions varies with the sparsity level of the pruned network. This observation indicates that an alternative similarity metric that takes into account features of the trained \cnets, such as sparsity level, could provide a gating function that offers improved performance on CIFAR-10-C corruptions.

\begin{table}[ht]
\small
\centering
\setlength\tabcolsep{4 pt}
\rotatebox{270}{
\begin{tabular}{c|l|ccc|ccc|cccc|cccc|cccc}
\toprule
{} &    {} & {} & {} & {}  & \multicolumn{3}{c}{Noise} & \multicolumn{4}{c}{Blur} & \multicolumn{4}{c}{Weather} & \multicolumn{4}{c}{Digital} \\
{} &                                              Model &  \rotatebox{90}{Clean} &  \rotatebox{90}{Avg. Robust} &  \rotatebox{90}{Memory (Mbit)} &\rotatebox{90}{gaussian} &   \rotatebox{90}{shot}&  \rotatebox{90}{impulse} &  \rotatebox{90}{defocus} &  \rotatebox{90}{glass} &  \rotatebox{90}{motion} &   \rotatebox{90}{zoom} &   \rotatebox{90}{snow} &  \rotatebox{90}{frost} &    \rotatebox{90}{fog} &  \rotatebox{90}{brightness} &  \rotatebox{90}{contrast} &  \rotatebox{90}{elastic} &  \rotatebox{90}{pixelate} &   \rotatebox{90}{jpeg} \\ 
\midrule
\multirow{3}{*}{\rotatebox[origin=c]{90}{Baseline}} & Dense (Augmix)  &          95.5 &          89.2 &     358 &   81.7 &          85.9 &          86.7 &          94.3 &          80.8 &          92.4 &          93.2 &          90.0 &          89.7 &          92.2 &          94.7 &          91.0 &          90.5 &          88.7 &          87.7 \\
         & Dense (Gaussian)  &          93.9 &          85.6  &     358 &     91.3 &          91.8 &          88.3 &          85.8 &          81.0 &          80.9 &          84.9 &          86.4 &          88.3 &          81.9 &          92.3 &          70.6 &          86.6 &          88.1 &          89.9 \\
         & Dense (Clean)  &          95.1 &          73.7 &   358 &       46.5 &          59.1 &          54.0 &          81.8 &          55.1 &          78.1 &          76.4 &          82.4 &          78.2 &          88.2 &          93.5 &          78.2 &          84.0 &          76.1 &          79.3 \\ \\\\[-2.2\medskipamount] \hline \\
\bottomrule
\end{tabular}
}
\vspace{2mm}
\caption{Performance when using ResNet-18 architecture with 80\% of weights pruned evaluated on CIFAR-10 and CIFAR-10-C. Note: LTH and LRR model prune rates are 79\%.}
\label{table:cards-80}
\end{table}

\begin{table}[ht]
\small
\centering
\setlength\tabcolsep{4 pt}
\rotatebox{270}{
\begin{tabular}{c|l|ccc|ccc|cccc|cccc|cccc}
\toprule
{} &    {} & {} & {} & {}  & \multicolumn{3}{c}{Noise} & \multicolumn{4}{c}{Blur} & \multicolumn{4}{c}{Weather} & \multicolumn{4}{c}{Digital} \\
{} &                                              Model &  \rotatebox{90}{Clean} &  \rotatebox{90}{Avg. Robust} &  \rotatebox{90}{Memory (Mbit)} &\rotatebox{90}{gaussian} &   \rotatebox{90}{shot}&  \rotatebox{90}{impulse} &  \rotatebox{90}{defocus} &  \rotatebox{90}{glass} &  \rotatebox{90}{motion} &   \rotatebox{90}{zoom} &   \rotatebox{90}{snow} &  \rotatebox{90}{frost} &    \rotatebox{90}{fog} &  \rotatebox{90}{brightness} &  \rotatebox{90}{contrast} &  \rotatebox{90}{elastic} &  \rotatebox{90}{pixelate} &   \rotatebox{90}{jpeg} \\ 
\midrule
\multirow{3}{*}{\rotatebox[origin=c]{90}{Baseline}} & Dense (Augmix)  &          95.5 &          89.2 &     358  &   81.7 &          85.9 &          86.7 &          94.3 &          80.8 &          92.4 &          93.2 &          90.0 &          89.7 &          92.2 &          94.7 &          91.0 &          90.5 &          88.7 &          87.7 \\
         & Dense (Gaussian)  &          93.9 &          85.6  &     358  &     91.3 &          91.8 &          88.3 &          85.8 &          81.0 &          80.9 &          84.9 &          86.4 &          88.3 &          81.9 &          92.3 &          70.6 &          86.6 &          88.1 &          89.9 \\
         & Dense (Clean)  &          95.1 &          73.7 &   358  &       46.5 &          59.1 &          54.0 &          81.8 &          55.1 &          78.1 &          76.4 &          82.4 &          78.2 &          88.2 &          93.5 &          78.2 &          84.0 &          76.1 &          79.3 \\ \\\\[-2.2\medskipamount] \hline \\
\bottomrule
\end{tabular}
}
\vspace{2mm}
\caption{ Performance when using ResNet-18 architecture with 90\% of weights pruned evaluated on CIFAR-10 and CIFAR-10-C. Note: LTH and LRR model prune rates are 89\%.}
\label{table:cards-90}
\end{table}

\begin{table}[ht]
\small
\centering
\setlength\tabcolsep{4 pt}
\rotatebox{270}{
\begin{tabular}{c|l|ccc|ccc|cccc|cccc|cccc}
\toprule
{} &    {} & {} & {} & {}  & \multicolumn{3}{c}{Noise} & \multicolumn{4}{c}{Blur} & \multicolumn{4}{c}{Weather} & \multicolumn{4}{c}{Digital} \\
{} &                                              Model &  \rotatebox{90}{Clean} &  \rotatebox{90}{Avg. Robust} &  \rotatebox{90}{Memory (Mbit)} &\rotatebox{90}{gaussian} &   \rotatebox{90}{shot}&  \rotatebox{90}{impulse} &  \rotatebox{90}{defocus} &  \rotatebox{90}{glass} &  \rotatebox{90}{motion} &   \rotatebox{90}{zoom} &   \rotatebox{90}{snow} &  \rotatebox{90}{frost} &    \rotatebox{90}{fog} &  \rotatebox{90}{brightness} &  \rotatebox{90}{contrast} &  \rotatebox{90}{elastic} &  \rotatebox{90}{pixelate} &   \rotatebox{90}{jpeg} \\ 
\midrule
\multirow{3}{*}{\rotatebox[origin=c]{90}{Baseline}} & Dense (Augmix)  &          95.5 &          89.2 &     358 &   81.7 &          85.9 &          86.7 &          94.3 &          80.8 &          92.4 &          93.2 &          90.0 &          89.7 &          92.2 &          94.7 &          91.0 &          90.5 &          88.7 &          87.7 \\
         & Dense (Gaussian)  &          93.9 &          85.6  &     358 &     91.3 &          91.8 &          88.3 &          85.8 &          81.0 &          80.9 &          84.9 &          86.4 &          88.3 &          81.9 &          92.3 &          70.6 &          86.6 &          88.1 &          89.9 \\
         & Dense (Clean)  &          95.1 &          73.7 &   358 &       46.5 &          59.1 &          54.0 &          81.8 &          55.1 &          78.1 &          76.4 &          82.4 &          78.2 &          88.2 &          93.5 &          78.2 &          84.0 &          76.1 &          79.3 \\ \\\\[-2.2\medskipamount] \hline \\
\bottomrule
\end{tabular}
}
\vspace{2mm}
\caption{ Performance when using ResNet-18 architecture with 95\% of weights pruned evaluated on CIFAR-10 and CIFAR-10-C.}
\label{table:cards-95}
\end{table}

\begin{table}[ht]
\small
\centering
\setlength\tabcolsep{4 pt}
\rotatebox{270}{
\begin{tabular}{c|l|ccc|ccc|cccc|cccc|cccc}
\toprule
{} &    {} & {} & {} & {}  & \multicolumn{3}{c}{Noise} & \multicolumn{4}{c}{Blur} & \multicolumn{4}{c}{Weather} & \multicolumn{4}{c}{Digital} \\
{} &                                              Model &  \rotatebox{90}{Clean} &  \rotatebox{90}{Avg. Robust} &  \rotatebox{90}{Memory (Mbit)} &\rotatebox{90}{gaussian} &   \rotatebox{90}{shot}&  \rotatebox{90}{impulse} &  \rotatebox{90}{defocus} &  \rotatebox{90}{glass} &  \rotatebox{90}{motion} &   \rotatebox{90}{zoom} &   \rotatebox{90}{snow} &  \rotatebox{90}{frost} &    \rotatebox{90}{fog} &  \rotatebox{90}{brightness} &  \rotatebox{90}{contrast} &  \rotatebox{90}{elastic} &  \rotatebox{90}{pixelate} &   \rotatebox{90}{jpeg} \\ 
\midrule
\multirow{3}{*}{\rotatebox[origin=c]{90}{Baseline}} & Dense (Augmix)  &          95.5 &          89.2 &     358  &   81.7 &          85.9 &          86.7 &          94.3 &          80.8 &          92.4 &          93.2 &          90.0 &          89.7 &          92.2 &          94.7 &          91.0 &          90.5 &          88.7 &          87.7 \\
         & Dense (Gaussian)  &          93.9 &          85.6  &     358   &     91.3 &          91.8 &          88.3 &          85.8 &          81.0 &          80.9 &          84.9 &          86.4 &          88.3 &          81.9 &          92.3 &          70.6 &          86.6 &          88.1 &          89.9 \\
         & Dense (Clean)  &          95.1 &          73.7 &   358  &       46.5 &          59.1 &          54.0 &          81.8 &          55.1 &          78.1 &          76.4 &          82.4 &          78.2 &          88.2 &          93.5 &          78.2 &          84.0 &          76.1 &          79.3 \\ \\\\[-2.2\medskipamount] \hline \\
\bottomrule
\end{tabular}
}
\vspace{2mm}
\caption{Performance of data-agnostic \cdeck ensembles composed of 80\% sparse ResNet-18 \cnets on individual CIFAR-10-C corruptions. For each corruption, entry is average over severity levels 1 through 5.}
\label{table:card-decks-80}
\end{table}

\begin{table}[ht]
\small
\centering
\setlength\tabcolsep{4 pt}
\rotatebox{270}{
\begin{tabular}{c|l|ccc|ccc|cccc|cccc|cccc}
\toprule
{} &    {} & {} & {} & {}  & \multicolumn{3}{c}{Noise} & \multicolumn{4}{c}{Blur} & \multicolumn{4}{c}{Weather} & \multicolumn{4}{c}{Digital} \\
{} &                                              Model &  \rotatebox{90}{Clean} &  \rotatebox{90}{Avg. Robust} &  \rotatebox{90}{Memory (Mbit)} &\rotatebox{90}{gaussian} &   \rotatebox{90}{shot}&  \rotatebox{90}{impulse} &  \rotatebox{90}{defocus} &  \rotatebox{90}{glass} &  \rotatebox{90}{motion} &   \rotatebox{90}{zoom} &   \rotatebox{90}{snow} &  \rotatebox{90}{frost} &    \rotatebox{90}{fog} &  \rotatebox{90}{brightness} &  \rotatebox{90}{contrast} &  \rotatebox{90}{elastic} &  \rotatebox{90}{pixelate} &   \rotatebox{90}{jpeg} \\ 
\midrule
\multirow{3}{*}{\rotatebox[origin=c]{90}{Baseline}} & Dense (Augmix)  &          95.5 &          89.2 &     358 &   81.7 &          85.9 &          86.7 &          94.3 &          80.8 &          92.4 &          93.2 &          90.0 &          89.7 &          92.2 &          94.7 &          91.0 &          90.5 &          88.7 &          87.7 \\
         & Dense (Gaussian)  &          93.9 &          85.6  &     358   &     91.3 &          91.8 &          88.3 &          85.8 &          81.0 &          80.9 &          84.9 &          86.4 &          88.3 &          81.9 &          92.3 &          70.6 &          86.6 &          88.1 &          89.9 \\
         & Dense (Clean)  &          95.1 &          73.7 &   358 &       46.5 &          59.1 &          54.0 &          81.8 &          55.1 &          78.1 &          76.4 &          82.4 &          78.2 &          88.2 &          93.5 &          78.2 &          84.0 &          76.1 &          79.3 \\ \\\\[-2.2\medskipamount] \hline \\
\bottomrule
\end{tabular}
}
\vspace{2mm}
\caption{Accuracy of data-agnostic \cdeck ensembles composed of 90\% sparse ResNet-18 \cnets on individual CIFAR-10-C corruptions. For each corruption, entry is average over severity levels 1 through 5.}
\label{table:card-decks-90}
\end{table}

\begin{table}[ht]
\small
\centering
\setlength\tabcolsep{4 pt}
\rotatebox{270}{
\begin{tabular}{c|l|ccc|ccc|cccc|cccc|cccc}
\toprule
{} &    {} & {} & {} & {}  & \multicolumn{3}{c}{Noise} & \multicolumn{4}{c}{Blur} & \multicolumn{4}{c}{Weather} & \multicolumn{4}{c}{Digital} \\
{} &                                              Model &  \rotatebox{90}{Clean} &  \rotatebox{90}{Avg. Robust} &  \rotatebox{90}{Memory (Mbit)} &\rotatebox{90}{gaussian} &   \rotatebox{90}{shot}&  \rotatebox{90}{impulse} &  \rotatebox{90}{defocus} &  \rotatebox{90}{glass} &  \rotatebox{90}{motion} &   \rotatebox{90}{zoom} &   \rotatebox{90}{snow} &  \rotatebox{90}{frost} &    \rotatebox{90}{fog} &  \rotatebox{90}{brightness} &  \rotatebox{90}{contrast} &  \rotatebox{90}{elastic} &  \rotatebox{90}{pixelate} &   \rotatebox{90}{jpeg} \\ 
\midrule
\multirow{3}{*}{\rotatebox[origin=c]{90}{Baseline}} & Dense (Augmix)  &          95.5 &          89.2 &     358  &   81.7 &          85.9 &          86.7 &          94.3 &          80.8 &          92.4 &          93.2 &          90.0 &          89.7 &          92.2 &          94.7 &          91.0 &          90.5 &          88.7 &          87.7 \\
         & Dense (Gaussian)  &          93.9 &          85.6  &     358   &     91.3 &          91.8 &          88.3 &          85.8 &          81.0 &          80.9 &          84.9 &          86.4 &          88.3 &          81.9 &          92.3 &          70.6 &          86.6 &          88.1 &          89.9 \\
         & Dense (Clean)  &          95.1 &          73.7 &   358  &       46.5 &          59.1 &          54.0 &          81.8 &          55.1 &          78.1 &          76.4 &          82.4 &          78.2 &          88.2 &          93.5 &          78.2 &          84.0 &          76.1 &          79.3 \\ \\\\[-2.2\medskipamount] \hline \\
\bottomrule
\end{tabular}
}
\vspace{2mm}
\caption{Accuracy of data-agnostic \cdeck ensembles composed of 95\% sparse ResNet-18 \cnets on individual CIFAR-10-C corruptions. For each corruption, entry is average over severity levels 1 through 5.}
\label{table:card-decks-95}
\end{table}

\begin{table}[ht]
\centering
\small
\setlength\tabcolsep{3 pt}
\begin{tabular}{ccccccccccc}
\toprule
{} & & \multicolumn{3}{c}{Baseline} & \multicolumn{6}{c}{ResNet-18} \\ \cmidrule(lr{4pt}){3-5}\cmidrule(lr){6-11}
{} & & \multicolumn{3}{c}{Dense} & \multicolumn{3}{c}{FT} & \multicolumn{3}{c}{GMP} \\ \cmidrule(lr){3-5}\cmidrule(lr){6-8}\cmidrule(lr){9-11}
{} & & \multicolumn{3}{c}{-} & \multicolumn{3}{c}{Global} & \multicolumn{3}{c}{Global} \\ \cmidrule(lr){3-5}\cmidrule(lr){6-8}\cmidrule(lr){9-11}
{} &  &      \rotatebox{90}{Augmix} &         \rotatebox{90}{Clean} &      \rotatebox{90}{Gauss.} &           \rotatebox{90}{Augmix} &         \rotatebox{90}{Clean} &      \rotatebox{90}{Gauss.} &           \rotatebox{90}{Augmix} &         \rotatebox{90}{Clean} &      \rotatebox{90}{Gauss.}  \\
\midrule
\multirow{3}{*}{\rotatebox[origin=c]{90}{\makecell{80\%}}} & Clean Acc.  &     95.5 &  95.1 &     93.9 &     95.6 &  94.6 &     94.1 &   95.1 &  94.3 &     93.4 \\
& Robust Acc. &     89.2 &  73.7 &     85.6 &     88.7 &  73.6 &     84.5 &   88.2 &    73 &     84.5 \\
& Memory (Mbit)     &         358 &        358 & 358 &   143 &  286 &  429 &  143 &  286 &  429 \\
\midrule
\multirow{3}{*}{\rotatebox[origin=c]{90}{\makecell{90\%}}} & Clean Acc.  &     95.5 &  95.1 &     93.9 &     94.7 &  94.1 &     93.1 &   94.8 &    94 &     93.4 \\
& Robust Acc. &     89.2 &  73.7 &     85.6 &     87.2 &  71.5 &     83.5 &   87.7 &  72.9 &     83.7 \\
& Memory (Mbit)      &         358  &        358 &           358 &  71.5 &  143 &  215 &  71.5 &  143 &  215 \\
\midrule
\multirow{3}{*}{\rotatebox[origin=c]{90}{\makecell{95\%}}} & Clean Acc.  &     95.5 &  95.1 &     93.9 &     93.8 &  93.4 &     92.2 &   94.6 &  93.9 &     92.9 \\
& Robust Acc. &     89.2 &  73.7 &     85.6 &     85.4 &    68 &     81.1 &   87.2 &    73 &     83.7 \\
& Memory (Mbit)     &         358 &        358 &           358 &  35.8 &  71.5 &  107 &  35.8 &  71.5 &  107 \\
\bottomrule
\end{tabular}
\vspace{2mm}
\caption{Performance comparison between dense baselines and ResNet-18 models pruned using fine-tuning (FT) and gradual magnitude pruning (GMP). Clean and Robust Acc. refer to accuracy on CIFAR-10 and CIFAR-10-C, respectively. }
\label{table:resnet18-ft-gmp}
\end{table}

\begin{table}[ht]
\tiny
\centering
\setlength\tabcolsep{3 pt}
\begin{tabular}{ccccccccccccccccccccccc}
\toprule
{} & & \multicolumn{3}{c}{Baseline} & \multicolumn{18}{c}{\cnet} \\ \cmidrule(lr{4pt}){3-5}\cmidrule(lr){6-23}
{} & & \multicolumn{3}{c}{Dense} & \multicolumn{6}{c}{Edgepopup} & \multicolumn{3}{c}{LRR} & \multicolumn{3}{c}{LTH} & \multicolumn{6}{c}{Biprop} \\ \cmidrule(lr){3-5}\cmidrule(lr){6-11}\cmidrule(lr){12-14}\cmidrule(lr){15-17}\cmidrule(lr){18-23}
{} & & \multicolumn{3}{c}{-} & \multicolumn{3}{c}{Layerwise} & \multicolumn{3}{c}{Global} & \multicolumn{3}{c}{Global} & \multicolumn{3}{c}{Global} & \multicolumn{3}{c}{Layerwise} & \multicolumn{3}{c}{Global}\\ \cmidrule(lr){3-5}\cmidrule(lr){6-8}\cmidrule(lr){9-11}\cmidrule(lr){12-14}\cmidrule(lr){15-17}\cmidrule(lr){18-20}\cmidrule(lr){21-23}
{} &  &      \rotatebox{90}{Augmix} &         \rotatebox{90}{Clean} &      \rotatebox{90}{Gauss.} &           \rotatebox{90}{Augmix} &         \rotatebox{90}{Clean} &      \rotatebox{90}{Gauss.} &           \rotatebox{90}{Augmix} &         \rotatebox{90}{Clean} &      \rotatebox{90}{Gauss.} &         \rotatebox{90}{Augmix} &         \rotatebox{90}{Clean} &      \rotatebox{90}{Gauss.} &         \rotatebox{90}{Augmix} &         \rotatebox{90}{Clean} &      \rotatebox{90}{Gauss.} &           \rotatebox{90}{Augmix} &         \rotatebox{90}{Clean} &      \rotatebox{90}{Gaussian} &          \rotatebox{90}{Augmix} &         \rotatebox{90}{Clean} &      \rotatebox{90}{Gauss.} \\
\midrule

\multirow{3}{*}{\rotatebox[origin=c]{90}{\makecell{90\%}}} & Clean Acc.  &     95.9 &  95.4 &     93.8 &          94.2 &         94.2 &              93 &             94.4 &              94 &               92.9 &   96.4 &  95.8 &     94.1 &   96.2 &   95.3 &       93.9 &          94.2 &           94 &            92.8 &             93.7 &            93.6 &               92.4 \\
\bottomrule
\end{tabular}
\vspace{2mm}
\caption{Performance comparison between dense baselines and \cnets for ResNeXt-29. Clean and Robust Acc. refer to accuracy on CIFAR-10 and CIFAR-10-C, respectively.}
\label{table:resnext29-card}
\end{table}

\begin{table}[ht]
\tiny
\centering
\setlength\tabcolsep{3 pt}
\begin{tabular}{ccccccccccccccccccccccc}
\toprule
{} & & \multicolumn{3}{c}{Baseline} & \multicolumn{18}{c}{\cnet} \\ \cmidrule(lr{4pt}){3-5}\cmidrule(lr){6-23}
{} & & \multicolumn{3}{c}{Dense} & \multicolumn{6}{c}{Edgepopup} & \multicolumn{3}{c}{LRR} & \multicolumn{3}{c}{LTH} & \multicolumn{6}{c}{Biprop} \\ \cmidrule(lr){3-5}\cmidrule(lr){6-11}\cmidrule(lr){12-14}\cmidrule(lr){15-17}\cmidrule(lr){18-23}
{} & & \multicolumn{3}{c}{-} & \multicolumn{3}{c}{Layerwise} & \multicolumn{3}{c}{Global} & \multicolumn{3}{c}{Global} & \multicolumn{3}{c}{Global} & \multicolumn{3}{c}{Layerwise} & \multicolumn{3}{c}{Global}\\ \cmidrule(lr){3-5}\cmidrule(lr){6-8}\cmidrule(lr){9-11}\cmidrule(lr){12-14}\cmidrule(lr){15-17}\cmidrule(lr){18-20}\cmidrule(lr){21-23}
{} &  &      \rotatebox{90}{Augmix} &         \rotatebox{90}{Clean} &      \rotatebox{90}{Gauss.} &           \rotatebox{90}{Augmix} &         \rotatebox{90}{Clean} &      \rotatebox{90}{Gauss.} &           \rotatebox{90}{Augmix} &         \rotatebox{90}{Clean} &      \rotatebox{90}{Gauss.} &         \rotatebox{90}{Augmix} &         \rotatebox{90}{Clean} &      \rotatebox{90}{Gauss.} &         \rotatebox{90}{Augmix} &         \rotatebox{90}{Clean} &      \rotatebox{90}{Gauss.} &           \rotatebox{90}{Augmix} &         \rotatebox{90}{Clean} &      \rotatebox{90}{Gaussian} &          \rotatebox{90}{Augmix} &         \rotatebox{90}{Clean} &      \rotatebox{90}{Gauss.} \\
\midrule

\multirow{3}{*}{\rotatebox[origin=c]{90}{\makecell{90\%}}} & Clean Acc.  &     95.9 &  95.1 &     93.6 &          94.8 &         94.4 &            93.3 &             95.4 &              95 &                 94 &   96.4 &  95.9 &     94.7 &   95.8 &    95.4 &     94.2 &          94.7 &         94.4 &              93 &             94.9 &            94.7 &               93.2 \\
\bottomrule
\end{tabular}
\vspace{2mm}
\caption{Performance comparison between dense baselines and \cnets for ResNet-50. Clean and Robust Acc. refer to accuracy on CIFAR-10 and CIFAR-10-C, respectively.}
\label{table:resnet50-card}
\end{table}

\begin{table}[ht]
\tiny
\centering
\setlength\tabcolsep{3 pt}
\begin{tabular}{ccccccccccccccccccccccc}
\toprule
{} & & \multicolumn{3}{c}{Baseline} & \multicolumn{18}{c}{\cnet} \\ \cmidrule(lr{4pt}){3-5}\cmidrule(lr){6-23}
{} & & \multicolumn{3}{c}{Dense} & \multicolumn{6}{c}{Edgepopup} & \multicolumn{3}{c}{LRR} & \multicolumn{3}{c}{LTH} & \multicolumn{6}{c}{Biprop} \\ \cmidrule(lr){3-5}\cmidrule(lr){6-11}\cmidrule(lr){12-14}\cmidrule(lr){15-17}\cmidrule(lr){18-23}
{} & & \multicolumn{3}{c}{-} & \multicolumn{3}{c}{Layerwise} & \multicolumn{3}{c}{Global} & \multicolumn{3}{c}{Global} & \multicolumn{3}{c}{Global} & \multicolumn{3}{c}{Layerwise} & \multicolumn{3}{c}{Global}\\ \cmidrule(lr){3-5}\cmidrule(lr){6-8}\cmidrule(lr){9-11}\cmidrule(lr){12-14}\cmidrule(lr){15-17}\cmidrule(lr){18-20}\cmidrule(lr){21-23}
{} &  &      \rotatebox{90}{Augmix} &         \rotatebox{90}{Clean} &      \rotatebox{90}{Gauss.} &           \rotatebox{90}{Augmix} &         \rotatebox{90}{Clean} &      \rotatebox{90}{Gauss.} &           \rotatebox{90}{Augmix} &         \rotatebox{90}{Clean} &      \rotatebox{90}{Gauss.} &         \rotatebox{90}{Augmix} &         \rotatebox{90}{Clean} &      \rotatebox{90}{Gauss.} &         \rotatebox{90}{Augmix} &         \rotatebox{90}{Clean} &      \rotatebox{90}{Gauss.} &           \rotatebox{90}{Augmix} &         \rotatebox{90}{Clean} &      \rotatebox{90}{Gaussian} &          \rotatebox{90}{Augmix} &         \rotatebox{90}{Clean} &      \rotatebox{90}{Gauss.} \\
\midrule

\multirow{3}{*}{\rotatebox[origin=c]{90}{\makecell{90\%}}} & Clean Acc.  &     95.6 &  95.2 &     93.3 &          94.8 &         94.2 &            92.4 &             95.0 &            94.7 &               92.8 &   
96.4 &  95.9 &     94 &   
96.1 &  95.6 &     94.2 &          94.5 &         94.4 &            92.7 &             94.6 &            94.1 &               92.6 \\
\bottomrule
\end{tabular}
\vspace{2mm}
\caption{Performance comparison between dense baselines and \cnets for WideResNet-18 (2x). Clean and Robust Acc. refer to accuracy on CIFAR-10 and CIFAR-10-C, respectively.}
\label{table:wideresnet18-card}
\end{table}

\begin{table}[ht]
\tiny
\centering
\setlength\tabcolsep{3 pt}
\begin{tabular}{ccccccccccccccccccccccc}
{} & {} & \multicolumn{6}{c}{\cdeck (Agnostic/Adaptive)} \\\cmidrule(lr){3-8}
{} & {}  & \multicolumn{3}{c}{Biprop (Layerwise)} & \multicolumn{3}{c}{Edgepopup (Layerwise)} \\\cmidrule(lr){3-5}\cmidrule(lr){6-8}
{} & {} &  2 & 4 & 6 &  2 & 4 & 6 \\
\midrule
\multirow{3}{*}{\rotatebox[origin=c]{90}{\makecell{90\%}}} & Clean Acc.  &              
92.3/         93.0 &
94.0/         94.5 &
94.8/\textbf{95.1} &
92.6/         93.3 &
94.3/         95.0 &
95.2/\textbf{95.4} \\
& Robust Acc. &                        
85.2/         86.2 &
88.4/         89.8 &
89.9/\textbf{90.5} &
85.3/         86.6 &
89.0/         90.4 &
90.6/\textbf{91.0} \\
& Memory (Mbit)     &                  
8.93 & 17.86 & 26.79 &           
8.93 & 17.86 & 26.79 \\
\midrule
\multirow{3}{*}{\rotatebox[origin=c]{90}{\makecell{95\%}}} & Clean Acc.  &                        
91.4/         92.0 &
93.3/         93.6 &
94.0/         94.1 &
92.1/         92.2 &
93.8/         94.1 &
94.7/94.8 \\
& Robust Acc. &                        
84.6/         85.2 &
87.4/         88.7 &
88.8/         89.3 &
85.0/         85.9 &
88.1/         89.3 &
89.7/90.0 \\
& Memory (Mbit)      &                   
4.46 & 8.93 & 13.39 &           
4.46 & 8.93 & 13.39 \\
\bottomrule
\end{tabular}
\vspace{2mm}
\caption{Performance of additional domain-agnostic and domain-adaptive \cdecks containing \cnets using WideResNet-18 architecture.}
\label{table:layerwise-agnostic-card-deck}
\end{table}

\begin{table}[ht]
\centering
\small
\setlength\tabcolsep{4 pt}
\begin{tabular}{l|ccc|cccc|cccc|cccc}
\toprule
{} & 
\multicolumn{3}{c}{Noise} & \multicolumn{4}{c}{Blur} & \multicolumn{4}{c}{Weather} & \multicolumn{4}{c}{Digital} \\
Augmentation &
\rotatebox{90}{gaussian} &   
\rotatebox{90}{shot} &  
\rotatebox{90}{impulse} &  
\rotatebox{90}{defocus} &  
\rotatebox{90}{glass} &  
\rotatebox{90}{motion} &   
\rotatebox{90}{zoom} &   
\rotatebox{90}{snow} &  
\rotatebox{90}{frost} & 
\rotatebox{90}{fog} &  
\rotatebox{90}{brightness} &  
\rotatebox{90}{contrast} &  
\rotatebox{90}{elastic} &  
\rotatebox{90}{pixelate} &   
\rotatebox{90}{jpeg} \\ 
\midrule
Augmix 
&   0   & 3.2   & \textbf{0}   
& \textbf{89}    & \textbf{0}  & \textbf{100}   & \textbf{96.2} 
& \textbf{100}   & \textbf{38.6}  & \textbf{100}   & \textbf{99.8} 
& \textbf{100} & \textbf{98.2} & \textbf{99.6} & 100 \\
Gaussian  
& \textbf{100} & \textbf{96.8}    & \textbf{100}   
& 11  & \textbf{100}    & 0    & 3.8   
& 0   & 61.4    & 0    & 0.2   
& 0   & 1.8     & 0.4   & \textbf{0}   \\
\bottomrule
\end{tabular}
\vspace{2mm}
\caption{{\bfseries Gating Function Selection by Corruption Type on CIFAR-10-C}: The \cdecks make use of models trained on Augmix and Gaussian augmented datasets. Here we provide the percentage of CIFAR-10-C test data that was gated to the Augmix and Gaussian models in the \cdeck based on the type of corruption. For each corruption type, the bold number indicates which \cnet achieves higher performance (on average) on that corruption. This highlights that our spectral similarity based gating function typically selects the best performing model. Note that for impulse and glass corruptions, the best performing models vary between those trained on Augmix and Gaussian corruptions based on the sparsity level of the model. For each corruption type, the percentage in the table gated to each augmentation method is an average over the gating percentages on CIFAR-10-C corruption severity levels 1 through 5.}
\label{table:gating-performance}
\end{table}

\section{Theory}
\label{appendix:theory}

\subsection{Proof of Corollary~\ref{thm:informal-card-deck}}
Using the triangle inequality, we have that
\begin{align}
\Big\| \sum_{k=1}^n \lambda_i f_k (x) - \sum_{k=1}^n \lambda_i F_k(\ell,\bm{w}) (x) \Big\|
&\leq \sum_{k=1}^n \lambda_i \| f_k (x) - F_k(\ell,\bm{w}) (x) \|, \label{eq:card-deck-1}
\end{align}
for any $x \in \mathcal{X}$. Hence, if 
\begin{align}
    \sup_{x \in \mathcal{X}} \| f_k(x) - F_k(\ell,\bm{w}) (x) \| \leq \varepsilon
    \label{eq:card-deck-2}
\end{align}
for each $1 \leq k \leq n$, then it immediately follows that
\begin{align}
\Big\| \sum_{k=1}^n \lambda_i f_k(x) - \sum_{k=1}^n \lambda_i F_k(\ell,\bm{w}) (x) \Big\|
&\leq \sum_{k=1}^n \lambda_i \varepsilon
= \varepsilon, \label{eq:card-deck-3}
\end{align}
for all $x \in \mathcal{X}$. Under the hypotheses of Theorem 3 in \citep{orseau2020logarithmic} (Theorem 2 in \citep{diffenderfer2021multiprize}), 
for each $k \in \{ 1, \ldots, n \}$ we have that with probability $(1 - \delta)$ there exists a full-precision (binary-weight)
\cnet satisfying (\ref{eq:card-deck-2}). Thus, with probability $(1 - \delta)^n$ there exists a collection of full-precision (binary-weight) networks $\{ f_k \}$ satisfying (\ref{result:card-deck-1}).

\subsection{OOD Robustness analysis}
To understand the average OOD robustness better, we derive the following decomposition:
\begin{eqnarray}
    Rob(\mathcal{D}_c, f^{a})&=& 
     Rob(\hat{S_a}, f^{a}) \nonumber\\
                 &+& [Rob(\mathcal{D}_a, f^{a}) - Rob(\hat{S_a}, f^{a})] \nonumber\\
                &+&[Rob(\mathcal{D}_c, f^{a})- Rob(\mathcal{D}_a, f^{a})]. \nonumber
\end{eqnarray}

Next, using the triangle inequality $a+(b-a) \geq a -\|b-a\|$ which is true because $\|b-a\|\geq a-b$ for $a, b\geq 0$, we have
\begin{eqnarray}
    Rob(\mathcal{D}_c, f^{a})&\geq& 
     Rob(\hat{S_a}, f^{a}) \nonumber\\
                 &-& \|Rob(\mathcal{D}_a, f^{a}) - Rob(\hat{S_a}, f^{a})\| \nonumber\\
                &-&\|Rob(\mathcal{D}_c, f^{a})- Rob(\mathcal{D}_a, f^{a})\|. \nonumber
\end{eqnarray}

By linearity of expectation, we can bound \eqref{avg_robust} from below  

 \begin{eqnarray}
Rob(\mathcal{D}^C, f^{Deck}) &\geq& \sum_{a=1}^{|A|}\sum_{c=1}^{|C|} w_c^a Rob(\hat{S_a}, f^{a})\label{empirical}\\
         &-&\sum_{a=1}^{|A|}\sum_{c=1}^{|C|} w_c^a\|Rob(\mathcal{D}_a, f^{a}) - Rob(\hat{S_a}, f^{a})\|  \label{gen-shift}\\
         &-&\sum_{a=1}^{|A|}\sum_{c=1}^{|C|} w_c^a\|Rob(\mathcal{D}_c, f^{a})- Rob(\mathcal{D}_a, f^{a})\|. \label{avg-dist-shift}
\end{eqnarray}
Note that we have bounded \eqref{avg_robust} in terms of the following three error terms for a classifier-corruption pair weighted by their gating (or selection) probabilities: 1) empirical robustness \eqref{empirical}, 2) generalization gap \eqref{gen-shift}, and 3) OOD-shift \eqref{avg-dist-shift}.

Next, we aim to provide a bound on the OOD-shift that is independent of the classifiers in hand and is only related to the properties of the distributions. To facilitate this, we define a notion of distance between two distributions.

\begin{Df}[ (Conditional Wasserstein distance]
\label{wasserstein}
For two labeled distributions $\mathcal{D}$ and $\mathcal{D}'$ with supports on $X\times Y$, we define conditional Wasserstein distance according to a distance metric $d$ as follows:
\begin{equation}
    \text{W}(\mathcal{D}, \mathcal{D}') = \E_{(.,y)\sim \mathcal{D}}\left[ \underset{J\in \mathcal{J}(\mathcal{D}|y, \mathcal{D}'|y)}{\inf}  \E_{(x,x')\sim J} d(x, x') \right],
\end{equation}
where $\mathcal{J}(\mathcal{D}, \mathcal{D}')$ is the set of joint distributions whose marginals are identical to $\mathcal{D}$ and $\mathcal{D}'$.
\end{Df}
Conditional Wasserstein distance between the two distributions is simply the expectation of Wasserstein distance between conditional distributions for each class.

\end{document}